\documentclass[letterpaper]{article} 
\usepackage{aaai25}  
\usepackage{times}  
\usepackage{helvet}  
\usepackage{courier}  
\usepackage[hyphens]{url}  
\usepackage{graphicx} 
\usepackage{natbib}  
\usepackage{caption} 
\usepackage{algorithm}
\usepackage{algorithmic}
\usepackage{amsmath}
\usepackage{subcaption}
\usepackage{amsthm}
\usepackage{comment}
\usepackage{xcolor}
\usepackage{amsfonts}
\usepackage{thm-restate}
\newtheorem{theorem}{Theorem}
\newtheorem{definition}{Definition}
\newtheorem{lemma}{Lemma}

\usepackage{newfloat}
\usepackage{listings}
\DeclareCaptionStyle{ruled}{labelfont=normalfont,labelsep=colon,strut=off} 
\floatstyle{ruled}
\newfloat{listing}{tb}{lst}{}
\floatname{listing}{Listing}
\title{Learning Robust and Privacy-Preserving Representations via Information Theory}
\author {
    Binghui Zhang\textsuperscript{\rm 1},
    Sayedeh Leila Noorbakhsh\textsuperscript{\rm 1},
    Yun Dong\textsuperscript{\rm 2},
    Yuan Hong\textsuperscript{\rm 3},
    Binghui Wang\textsuperscript{\rm 1}
}
\affiliations {
    \textsuperscript{\rm 1}Illinois Institute of Technology\\
    \textsuperscript{\rm 2}Milwaukee School of Engineering\\
    \textsuperscript{\rm 3}University of Connecticut\\
    \{bzhang57,snoorbakhsh\}@hawk.iit.edu, dong@msoe.edu, yuan.hong@uconn.edu, bwang70@iit.edu
}

\begin{document}

\maketitle

\begin{abstract}
Machine learning models are vulnerable to both security attacks (e.g., adversarial examples) and privacy attacks  (e.g., private attribute inference). 
We take the first step to mitigate both the security and privacy attacks, and maintain task utility as well.
Particularly, we propose an information-theoretic framework to achieve the goals through the lens of representation learning, i.e., learning  representations that are robust to both  adversarial examples and attribute inference adversaries. 
We also derive novel theoretical results under our framework, e.g., the inherent trade-off between adversarial robustness/utility and attribute privacy, and guaranteed attribute privacy leakage against attribute inference adversaries. 
\end{abstract}

\begin{links}
    \link{Code\&Full Report}{https://github.com/ARPRL/ARPRL}
\end{links}

\section{Introduction}
\label{sec:intro}

Machine learning (ML) has achieved breakthroughs in
many areas such as computer vision and natural language processing. 
However, recent works show current ML design is vulnerable to both security and privacy attacks, e.g., adversarial examples and private attribute inference. 
Adversarial examples~\citep{szegedy2013intriguing,CarliniSP17,qu2023certified,hong2024certifiable}, i.e., natural data with imperceptible perturbations, cause ML models to make incorrect predictions and 
prevent them from being deployed in safety-critical applications such as autonomous driving~\citep{eykholt2018robust} and medical imaging~\citep{bortsova2021adversarial}. 
Many real-world applications involve data containing private information, such as race, gender, and income. When applying ML to these applications, it poses a great challenge as private attributes can often be accurately inferred~\citep{AttriInferWWW17,aono2017privacy,melis2019exploiting}. 

To mitigate adversarial examples and attribute inference attacks, many defenses are proposed but follow two separate lines with different techniques.  
For instance, state-of-the-art defenses against adversarial examples are based on adversarial training~\citep{madry2018towards,zhang2019theoretically,wang2019improving}, which solves a min-max optimization problem.  
In contrast, a representative defense against inference attacks is based on differential privacy~\citep{abadi2016deep}, which is a statistical method. 
Further, some works~\citep{song2019privacy,song2019membership} show adversarially robust models can even leak more private information (also verified in our results).

In this paper, we focus on the research question: 
\emph{1) Can we design a model that ensures adversarial robustness and attribute privacy protection while maintaining the utility of any (unknown) downstream tasks simultaneously? 2) Further, can we theoretically understand the relationships among adversarial robustness, utility, and attribute privacy?}
To achieve the goal, we propose an information-theoretic defense framework  through the lens of \emph{representation learning}, termed {\bf ARPRL}. 
Particularly, instead of training models from scratch, which requires huge computational resources and is time consuming, shared learnt representations ensures the community to save much time and costs for future use. 
Our ARPRL is partly inspired by two  works \citep{zhu2020learning,zhou2022improving}, which show  adversarially robust representations based defenses outperform the \emph{de facto} adversarial training based methods, while being \emph{the first work} to contrivially generalize learning data representations that are robust to both adversarial examples and attribute inference adversaries. More specifically, we formulate learning representations via three mutual information (MI) objectives: one for adversarial robustness, our for attribute privacy protection, and one for utility preservation. 
We point out that our ARPRL is \emph{task-agnostic}, meaning the learnt representations does not need to know the target task at hand and can be used for any downstream task. 
Further, based on our MI objectives, we can derive several theoretical results. For instance, we obtain an inherent tradeoff between adversarial robustness and attribute privacy, as well as between utility and attribute privacy. These tradeoffs are also verified through the experimental evaluations on multiple benchmark datasets. We also derive the guaranteed attribute privacy leakage. Our key contributions are as below:

\begin{itemize}
\item 
We propose the first information-theoretic framework to unify both adversarial robustness and privacy protection.
\item We formulate learning  adversarially robust and privacy-preserving
representations via mutual information goals and train neural networks to approximate them. 
\item We provide novel theoretical results: the tradeoff between adversarial robustness/utility and attribute privacy, and guaranteed attribute privacy leakage. 
\end{itemize}

\section{Preliminaries and Problem Setup}
\label{sec:background}

{\bf Notations.} We use ${s}$, ${\bf s}$, and $\mathcal{S}$ to denote (random) scalar, vector, and space, respectively. 
Given a data  ${\bf x} \in \mathcal{X} $, we denote  its label as  $y \in \mathcal{Y}$ and private attribute as $u \in \mathcal{U}$, where $\mathcal{X}$, $\mathcal{Y}$, and $\mathcal{U}$ are input data space, label space, and attribute space, respectively. 
An $l_p$ ball centered at a data  ${\bf x}$ with radius $\epsilon$ is defined as $\mathcal{B}_p({\bf x}, \epsilon) = \{{\bf x}' \in \mathcal{X}: \|{\bf x}'-{\bf x}\|_p \leq \epsilon \}$. 
The joint distribution of ${\bf x}$, $y$, and $u$ is denoted as $\mathcal{D}$.  
We further denote 
$f: \mathcal{X} \rightarrow \mathcal{Z}$ 
as the representation learner 
that maps ${\bf x} \in \mathcal{X}$ to its representation ${\bf z} \in \mathcal{Z}$, where $\mathcal{Z}$ is the representation space. 
We let $C: \mathcal{Z} \rightarrow \mathcal{Y}$ be the \emph{primary task classifier}, which predicts label $y$ based on ${\bf z}$, and $A: \mathcal{Z} \rightarrow \mathcal{U}$ be the \emph{attribute inference classifier}, which infers $u$ based on the representation ${\bf z}$. The mutual information (MI) of two random variables ${\bf x}$ and  ${\bf z}$ is denoted by $I({\bf x}; {\bf z})$. 

\noindent  {\bf Adversarial example/perturbation, adversarial risk, and representation vulnerability (Zhu et al., 2020)}.
Let $\epsilon$ be the $l_p$ perturbation budget.
For any classifier $C: \mathcal{X} \rightarrow \mathcal{Y}$, the \emph{adversarial risk}
of $C$ with respect to 
$\epsilon$ is defined as: 
\begin{align}
\label{eqn:advrisk}
\textrm{AdvRisk}_\epsilon(C) & = \textrm{Pr} [\exists {\bf x}' \in \mathcal{B}_p({\bf x}, \epsilon), \, \textrm{s.t. } C({\bf x}') \neq y] \notag \\
& = \sup\nolimits_{{\bf x}' \in \mathcal{B}_p({\bf x}, \epsilon)} \textrm{Pr} [C({\bf x}') \neq y],
\end{align}
where ${\bf x}'$ is called  \emph{adversarial example} and $\delta = {\bf x}' - {\bf x}$ is  \emph{adversarial  perturbation} with an $l_p$ budget $\epsilon$, i.e., $\|\delta\|_p \leq \epsilon$. 
Formally, adversarial risk captures the vulnerability of a classifier to adversarial perturbations.
When  $\epsilon = 0$, adversarial risk reduces to the standard risk, i.e., $\textrm{AdvRisk}_0(C) = \textrm{Risk}(C) = \textrm{Pr}(C({\bf x}) \neq y)$. 
Motivated by the empirical and theoretical difficulties of robust learning with adversarial examples, \citet{zhu2020learning,zhou2022improving} target learning adversarially
robust representations based on mutual information. 
They introduced the term \emph{representation vulnerability} as follow: 
Given a representation learner $f: \mathcal{X}  \rightarrow \mathcal{Z}$ and an $l_p$ perturbation budget $\epsilon$,  the representation
vulnerability of $f$ with respect to $\epsilon$ is defined as
\begin{align}
     \mathrm{RV}_\epsilon ({f}) & = \max\nolimits_{{\bf x}' \in \mathcal{B}_p({\bf x}, \epsilon)} [I({\bf x}; {\bf z}) - I({\bf x}';  {\bf z}')],
\end{align}
where ${\bf z} = f({\bf x})$ and ${\bf z}' = f({\bf x}')$
are the learnt representation for ${\bf x}$ and ${\bf x}'$, respectively.
We note  \emph{higher/smaller $\mathrm{RV}_\epsilon ({f})$ values imply the representation is less/more robust to adversarial perturbations}. 
Further, \cite{zhu2020learning} linked the connection between adversarial robustness and representation vulnerability through the following theorem: 
\begin{theorem}
\label{thm:rob_rep}
Consider all primary task classifiers as $\mathcal{C} = \{ C: \mathcal{Z} \rightarrow \mathcal{Y} \}$.
Given the perturbation budget $\epsilon$, for any representation learner $f: \mathcal{X} \rightarrow \mathcal{Z}$, 
\begin{small}
\begin{align}
    \inf_{C \in \mathcal{C}} \mathrm{AdvRisk}_\epsilon (C \circ f) & \geq 
    1 - \frac{\big({ I({\bf x}; {\bf z}) - \mathrm{RV}_\epsilon (f) + \log 2}\big)}{\log |\mathcal{Y}|}. \label{eqn:urtradeoff}
\end{align}
\end{small}

\end{theorem}
The theorem states that a smaller representation vulnerability implies a smaller lower bounded adversarial risk, which means better adversarial robustness, and vice versa.  
Finally, $f$ is called $(\epsilon,\tau)$-robust if $\mathrm{RV}_\epsilon (f) \leq \tau$. 

\noindent  {\bf Attribute inference attacks and advantage.} Following existing privacy analysis~\cite{salem2023sok}, we assume the private attribute space $\mathcal{U}$ is binary.
Let $\mathcal{A}$ be the set of all binary attribute inference classifiers, 
i.e. $\mathcal{A}=\{A: \mathcal{Z}
\rightarrow \mathcal{U} = \{0,1\}\}$. 
Then, we formally define the \emph{attribute inference advantage} of the {worst-case} attribute inference adversary with respect to the joint distribution $\mathcal{D} = \{{\bf x}, y, u\}$ as below:
\begin{align}
    \label{eqn:adv}
    & \textrm{Adv}_{\mathcal{D}} (\mathcal{A}) = \max_{A \in \mathcal{A}} | \textrm{Pr}_{\mathcal{D}}(A({\bf z}) =a | u=a) \notag \\
    &- \textrm{Pr}_{\mathcal{D}} (A({\bf z}) =a | u=1-a) |, \, \forall a=\{0,1\}. 
\end{align}
We can observe that: if $\textrm{Adv}_{\mathcal{D}}(\mathcal{A}) =1$, an adversary can \emph{completely} infer the privacy attribute through the learnt representations. In contrast, if $\textrm{Adv}_{\mathcal{D}}(\mathcal{A}) =0$, an adversary obtains a \emph{random guessing} inference performance. To protect the private attribute, we aim to obtain a small $\textrm{Adv}_{\mathcal{D}}$.

\noindent  {\bf Threat model and problem setup.} 
We consider an attacker performing both attribute inference and adversarial example attacks. \emph{We assume the attacker does not have the access to the representation learner (i.e., $f$), but can obtain the shared data representations.
} 
Our goal is to learn task-agnostic representations that are adversarially robust, protect attribute privacy, and  maintain the utility of any downstream task. 
Formally, given data $\{{\bf x}, y, u\}$ from an underlying distribution $\mathcal{D}$, and a perturbation budget $\epsilon$, we aim to obtain the representation learner $f$ such that the representation vulnerability $RV_\epsilon(f)$ is small, attribute inference advantage $\textrm{Adv}_\mathcal{D}(\mathcal{A})$ is small, but the performance is high, i.e., $\textrm{Risk}(C)$ is small.

\section{Design of ARPRL}
\label{sec:design}

In this section, we will design our {\bf a}dversarilly {\bf r}obust and {\bf p}rivacy-preserving {\bf r}epresentation {\bf l}earning method, termed {\bf ARPRL}, inspired by information theory.
\subsection{Formulating ARPRL via MI Objectives}

Given a data ${\bf x}$ with private attribute $u$ sampled from a distribution $ \mathcal{D}$, and a perturbation budget $\epsilon$, we aim to learn the representation ${\bf z} = f({\bf x})$ for ${\bf x}$ that satisfies three goals: 

\begin{itemize}
\item \textbf{Goal 1: Protect attribute privacy.} 
${\bf z}$ contains as less information as possible about 
the private attribute 
$u$. Ideally, when ${\bf z}$ does not include information about 
$u$, i.e., ${\bf z} \perp u$, it is impossible to infer 
$u$ from 
${\bf z}$. 

\item \textbf{Goal 2: Preserve utility.} 
${\bf z}$ should be be useful for many downstream tasks. Hence it should include as much information about 
${\bf x}$ as possible, while excluding the private 
$u$. Ideally, when 
${\bf z}$ retains the most information about 
${\bf x}$, the model trained on 
${\bf z}$ will have the same performance as the model trained on the raw 
${\bf x}$ (though we do not know the downstream task), thus preserving utility.

\item \textbf{Goal 3: Adversarially robust.} 
${\bf z}$ should be not sensitive to adversarial perturbations on the data ${\bf x}$, indicating a small representation vulnerability.
\end{itemize}

We propose to formalize the above goals via MI. 
Formally, we quantify the goals as below:

\begin{align}
    & \textbf{Formalize Goal 1: }\quad \min_f I({\bf z}; u); \label{Eqn:minpriv} \\
    & \textbf{Formalize Goal 2: } \quad \max_f I({\bf x}; {\bf z} | u);  
    \label{Eqn:maxutil} \\ 
     & \textbf{Formalize Goal 3: }     \label{Eqn:RV}  \\ & \quad 
    \min_f  
    RV_{\epsilon}(f|u) = 
        \min_f \max_{{\bf x}' \in \mathcal{B}_p({\bf x}, \epsilon)} I({\bf x}; {\bf z} | u) -  I({\bf x}'; {\bf z}' | u). 
     \notag
\end{align}
where 
1) 
we minimize $I({\bf z}; u)$ to maximally reduce the correlation between ${\bf z}$ and the private attribute $u$; 
2) 
we maximize $ I({\bf x}; {\bf z} | u) $ to keep the raw information in ${\bf x}$ as much as possible in ${\bf z}$ while excluding information about the private $u$; 
3) $RV_{\epsilon}(f|u)$ is the representation vulnerability of $f$ conditional on $u$ with respect to $\epsilon$.  
Minimizing it learns adversarially robust representations that exclude the information about private $u$. 
Note that $I({\bf x}; {\bf z} | u) $ in Equation (\ref{Eqn:RV}) can be merged with that in Equation (\ref{Eqn:maxutil}). Hence Equation (\ref{Eqn:RV}) can be reduced to the below min-max optimization problem: 
\begin{align}
    \max_f \min_{{\bf x}' \in \mathcal{B}_p({\bf x}, \epsilon)} I({\bf x}'; {\bf z}' | u).
    \label{Eqn:advrobust} 
\end{align}

\subsection{Estimating MI via Variational Bounds}
The key challenge of solving the above MI objectives is that calculating an MI between two arbitrary random variables is likely to be infeasible \citep{peng2018variational}. 
To address it, we are inspired by the existing MI neural estimation methods~\citep{alemi2017deep,belghazi2018mutual,poole2019variational,hjelm2019learning,cheng2020club}, which convert the intractable exact MI calculations to the tractable variational MI bounds.  
Then, we parameterize each variational MI bound with a neural network, and train the neural networks to approximate the true MI. \emph{We clarify that we do not design new MI neural estimators, but adopt existing ones to aid our customized MI terms for learning adversarially robust and privacy-preserving representations.} 

\noindent {\bf Minimizing upper bound MI in Equation (\ref{Eqn:minpriv}) for privacy protection.}
We adapt the variational upper bound CLUB proposed in~\cite{cheng2020club}. Specifically, via some derivations (see Appendix B), we have 
\begin{align*}
& I({\bf z};u) 
 \leq \min I_{vCLUB}({\bf z};u) \Longleftrightarrow  \max \mathbb{E}_{p({\bf z}, u)} [\log q_{\Psi}(u|{\bf z}) ]
\end{align*}
where 
$q_{\Psi}(u | {\bf z})$ is an auxiliary posterior distribution of $p(u | {\bf z})$. 
Then our {\bf Goal 1} for privacy protection can be reformulated as solving the 
min-max objective function below: 
\begin{align*}
        & \min_f \min \limits_{\Psi}  I_{vCLUB}({\bf z};u)
         \Longleftrightarrow \min_f \max \limits_{\Psi} 
        \mathbb{E}_{p({\bf z}, u)} [\log q_{\Psi}(u| {\bf z})]
\end{align*}
\noindent \emph{Remark.} 
This equation can be interpreted as an \emph{adversarial game} between: (1) an adversary $q_{\Psi}$ (i.e., attribute inference classifier) who aims to infer the private attribute $u$ from the representation ${\bf z}$; and (2) a defender (i.e., the representation learner $f$) who aims to protect $u$  from being inferred. 

\noindent {\bf Maximizing lower bound MI in Equation (\ref{Eqn:maxutil}) for utility preservation.}
We adopt the MI estimator proposed in~\cite{nowozin2016f} to estimate the lower bound of 
the MI Equation (\ref{Eqn:maxutil}). 
Via some derivations 
(see details in Appendix B) we have:
\begin{align*}
    I({\bf x} ; {\bf z} |u ) 
    \geq H({\bf x} | u) + \mathbb{E}_{p({\bf x} , {\bf z}, u)} [\log q_\Omega ({\bf x} | {\bf z}, u)],  
\end{align*}
where $q_{\Omega}$ is an \emph{arbitrary} 
auxiliary posterior distribution. 
Since $H({\bf x}|u)$ is a constant, our {\bf Goal 2} 
can be rewritten as the below max-max objective function:
\begin{align*}
    & \max_f I({\bf x}; {\bf z} | u) 
    \iff 
    \max_{f, \Omega} \mathbb{E}_{p({\bf x}, {\bf z}, u)} \left[\log {q_{\Omega}}({\bf x} | {\bf z}, u)  \right]. 
\end{align*}
\noindent \emph{Remark.}
This equation can be interpreted as a \emph{cooperative game} between the representation learner $f$ and $q_{\Omega}$ who aim to preserve the utility collaboratively. 

\noindent {\bf Maximizing the worst-case MI in Equation (\ref{Eqn:advrobust}) for adversarial robustness.}
To solve Equation (\ref{Eqn:advrobust}), one needs to first find the perturbed data ${\bf x}' \in \mathcal{B}_p({\bf x}, \epsilon)$ that minimizes MI $I({\bf x}'; {\bf z}'|u)$, and then maximizes this MI by training the representation learner $f$.  
As claimed in~\cite{zhu2020learning,zhou2022improving}, minimizing the MI on the worst-case perturbed data is computational challenging. An approximate solution~\cite{zhou2022improving} is first performing a strong white-box attack, e.g., the projected gradient descent (PGD) attack~\cite{madry2018towards}, to generate a set of adversarial examples, and then selecting the adversarial example that has the smallest MI. 
Assume the strongest adversarial example is ${\bf x}^{a} = \arg\min_{{\bf x}' \in \mathcal{B}_p({\bf x}, \epsilon)} I({\bf x}'; {\bf z}' | u)$. The next step is to maximize the MI $I({\bf x}^{a}; {\bf z}^{a} | u)$. \citet{zhu2020learning} used the MI Neural Estimator (MINE)~\citep{belghazi2018mutual} to estimate this MI. Specifically, 
\begin{align*}
& I({\bf x}^{a}; {\bf z}^{a} | u) \geq 
I_{\Lambda} ({\bf x}^{a}; {\bf z}^{a} | u) = \mathbb{E}_{p({\bf x}^{a}, {\bf z}^{a}, u)} [t_{\Lambda}({\bf x}^{a}, {\bf z}^{a}, u)] 
\notag \\
& \qquad \quad 
- \log \mathbb{E}_{p({\bf x}^{a})p({\bf z}^{a})p(u)}[\exp(t_{\Lambda}({\bf x}^{a}, {\bf z}^{a}, u))], 
\end{align*}
where $t_{\Lambda}: \mathcal{X} \times \mathcal{Z} \times \{0,1\} \rightarrow \mathbb{R}$ can be any family of neural networks parameterized with $\Lambda$. 
More details about calculating the MI are deferred to the implementation section. 

\noindent {\bf Objective function of ARPRL.} By 
using the above MI bounds, 
our objective function of ARPRL is as follows:
{
\small
\begin{align}
    & \max_f \Big( \alpha  \min_{\Psi} -\mathbb{E}_{p({\bf x},u)} \left[ \log q_{\Psi}(u | f({\bf x})) \right] 
    + \beta \max_{\Lambda} I_{\Lambda} ({\bf x}^{a}; {\bf z}^{a} | u)
    \notag \\ 
    & \qquad + (1-\alpha-\beta) \big( \max_{\Omega} \mathbb{E}_{p({\bf x}, u)} \left[ \log q_{\Omega}({\bf x} | f({\bf x}), u) \right] \notag \\ 
    & \qquad \qquad +  \lambda \max_{\Delta} \mathbb{E}_{p({\bf x}, y)} \left[ \log q_\Delta (y| f ({\bf x})) \right]  \big) \Big) 
    \label{eq:final-obj}, 
\end{align}
}
where $\alpha, \beta \in [0,1]$ tradeoff between privacy and utility, and robustness and utility, respectively. That is, a larger/smaller $\alpha$ indicates a stronger/weaker attribute privacy protection and a larger/smaller $\beta$ indicates a stronger/weaker robustness against adversarial perturbations. 

\begin{figure}[!t]
    \centering
    \includegraphics[width=0.42\textwidth]{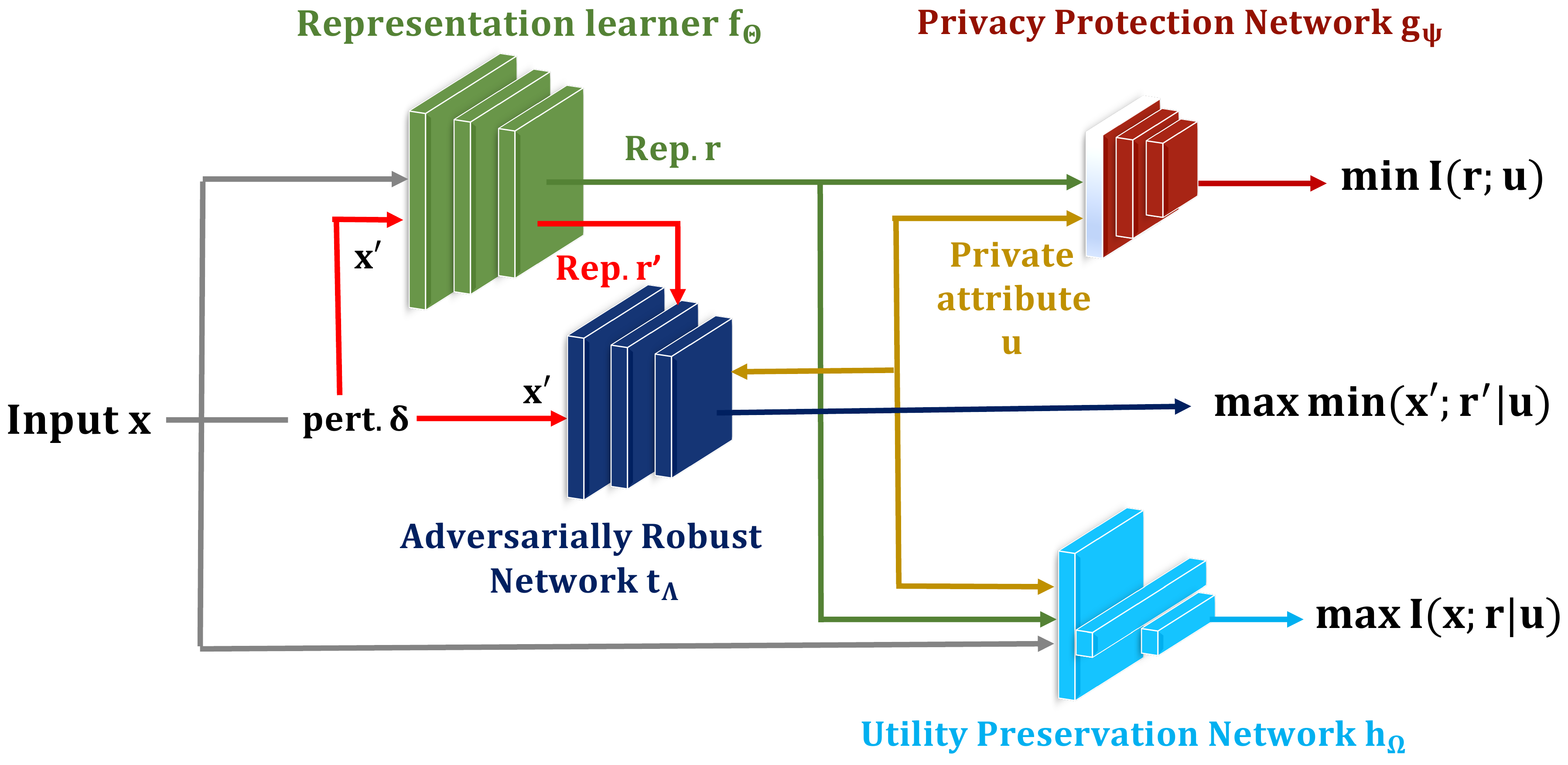}
    \caption{Overview of ARPRL.}
    \label{fig:diagram}
\end{figure}

\subsection{Implementation in Practice via Training Parameterized Neural Networks}
\label{sec:implementation}
In practice, Equation (\ref{eq:final-obj}) is solved via training four neural networks, i.e., the representation learner $f_\Theta$ (parameterized with $\Theta$),  privacy-protection network $g_{\Psi}$ associated with the auxiliary distribution $q_{\Psi}$, robustness network $t_{\Lambda}$ associated with the MINE estimator, and utility-preservation network  $h_{\Omega}$ associated with the auxiliary distribution $q_{\Omega}$, on a set of training data. 
Suppose we have collected a set of samples $\{({\bf x}_j, y_j, u_j)\}$ 
from the dataset distribution $\mathcal{D}$. 
We can then approximate each term in Equation (\ref{eq:final-obj}). 

Specifically, we approximate the expectation associated with the privacy-protection network network  $g_{\Psi}$ as  
{
\begin{align*}
\mathbb{E}_{p(u, {\bf x})}\log q_{\Psi}(u|f({\bf x})))\approx
 -\sum\nolimits_{j} CE(u_j,g_{\Psi}(f({\bf x}_j)))
 \end{align*} 
 }
 where $CE(\cdot)$ means the cross-entropy loss function. 
  
 Further, we approximate the expectation associated with the utility-preservation network  $h_{\Omega}$ via the \textit{Jensen-Shannon} (JS) MI estimator \cite{hjelm2019learning}. That is,
{
\small
\begin{align*}
& \mathop{\mathbb{E}}_{p({\bf x}, u)}\log q_{\Omega}({\bf x}|f({\bf x}),u) \approx I^{(JS)}_{\Theta,\Omega} ({\bf x}; f({\bf x}),u) 
\notag \\
& 
=\mathop{\mathbb{E}}_{p({\bf x},u)}[-\textrm{sp}(-h_{\Omega}({\bf x},f({\bf x}),u)]-\mathop{\mathbb{E}}_{p({\bf x},u,\bar{{\bf x}})}[\textrm{sp}(h_{\Omega}(\bar{{\bf x}},f({\bf x}),u)]
\end{align*}
}
where $\bar{\bf x}$ is an independent random sample of the same distribution as ${\bf x}$, and expectation can be replaced by samples $\{{\bf x}_j^i, \bar{\bf x}_j^i, u_j^i \}$. $\textrm{sp}(z) = \log(1+\exp(z))$ is a softplus function. 

Regarding the MI related to the robustness network $t_{\Lambda}$,  
we can adopt the methods proposed in \cite{zhu2020learning,zhou2022improving}. 
For instance, \cite{zhu2020learning} proposed to avoid searching the whole ball, and restrict the search space to be the set of
empirical distributions with, e.g., $m$ samples: 
$\mathcal{S}_m(\epsilon) = \{\frac{1}{m} \sum_{i=1}^m \delta_{{\bf x}'_i}: {\bf x}'_i \in \mathcal{B}_p({\bf x}_i, \epsilon), \forall i \in [m] \}$. 
Then it estimates the MI $\min_{{\bf x}' \in \mathcal{B}_p({\bf x}, \epsilon)} I({\bf x}'; f({\bf x}') | u)$ as 
{
\begin{align}
    \min_{{\bf x}'} I_{\Lambda}^{(m)} ({\bf x}';  f({\bf x}') | u) \, \textrm{ s.t. } {\bf x}' \in \mathcal{S}_m (\epsilon),
    \label{eqn:worst_mine}
\end{align}
}
where $I_{\Lambda}^{(m)} ({\bf x}'; f({\bf x}') | u) = \frac{1}{m} \sum_{i=1}^m t_{\Lambda}({\bf x}_i, f({\bf x}_i), u_i) - \log [\frac{1}{m} \sum_{i=1}^m e^{t_{\Lambda}({\bar{\bf x}}_i, f({\bf x}_i), u_i)}]$, where $\{\bar{\bf x}_i\}$ are independent and random samples that have the same distribution as $\{{\bf x}_i\}$.  

Zhu et al. 2020 propose an alternating minimization algorithm to solve Equation (\ref{eqn:worst_mine})---it alternatively performs gradient ascent on $\Lambda$ to maximize $I_{\Lambda}^{(m)}({\bf x}'; f({{\bf x}'})|u)$ given $\mathcal{S}_m(\epsilon)$, and then searches for the set of worst-case perturbations on $\{{\bf x}_i': i \in [m]\}$ given $\Lambda$ based on, e.g., projected gradient descent. Figure \ref{fig:diagram} overviews our ARPRL. Algorithm 1 in Appendix details the training of ARPRL.

\begin{figure*}[!t]
	\centering
	\subfloat[Raw circle data]
	{\centering\includegraphics[scale=0.23]{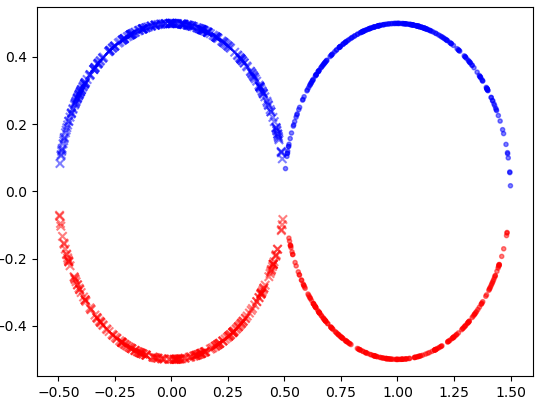}}
	\subfloat[$\alpha=0,\beta=0.5$]
	{\centering \includegraphics[scale=0.23]{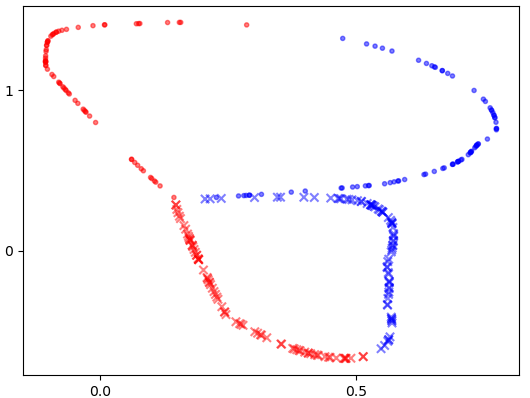}}
	\subfloat[$\alpha=0.1,\beta=0.5$]
	{\centering \includegraphics[scale=0.23]{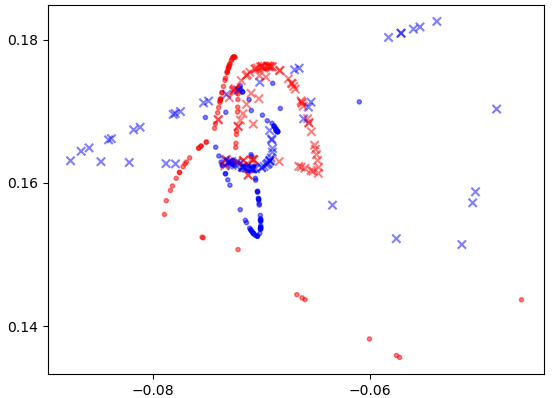}}
	\caption{2D representations learnt by ARPRL. (a) Raw data; (b) only robust representations (privacy acc: 99\%, robust acc: 88\%, test acc: 99\%); and (c) robust + privacy preserving representations (privacy acc: 55\%, robust acc: 75\%, test acc: 85\%). {\color{red} red} vs. {\color{blue} blue}: binary private attribute values;  cross $\times$ vs. circle $\circ$: binary task labels.} 
	\label{fig:toy2D}
\end{figure*}

\subsection{Theoretical Results
\footnote{\cite{zhao2020trade} also has theoretical results of privacy protection against attribute inference attacks. The differences between theirs and our theoretical results are discussed in the Appendix.}}

\noindent  {\bf Robustness vs. Representation Vulnerability.}
We first show the relationship between adversarial risk (or robustness) and representation vulnerability in ARPRL. 
\begin{restatable}[]{theorem}{urtradeoff}
\label{thm:urtradeoff}

Let all binary task classifiers be $\mathcal{C} = \{ C: \mathcal{Z} \rightarrow \mathcal{Y} \}$.
Then for any representation learner $f: \mathcal{X} \rightarrow \mathcal{Z}$, 
{
\begin{align}
    \inf_{C \in \mathcal{C}} \mathrm{AdvRisk}_\epsilon (C \circ f) & \geq 
    \frac{\mathrm{RV}_\epsilon (f|u) - I({\bf x}; {\bf z} | u)}{\log 2}.
        \label{eqn:urtradeoff}
\end{align}
}
\end{restatable}

\noindent \emph{Remark.} 
Similar to Theorem~\ref{thm:rob_rep}, Theorem~\ref{thm:urtradeoff} shows a smaller representation vulnerability implies a smaller lower bounded adversarial risk for robustness. 
In addition, a larger MI $I({\bf x}; {\bf z} | u)$ ({\bf Goal 2} for utility preservation) produces a smaller adversarial risk, implying better robustness. 

\noindent {\bf Utility vs. Privacy Tradeoff.}
The following theorem shows the tradeoff between utility and privacy: 
\begin{restatable}[]{theorem}{uptradeoff}
\label{thm:uptradeoff}
Let ${\bf z} = f({\bf x})$ be with a bounded norm $R$ 
(i.e., $\max_{{\bf z} \in \mathcal{Z}} \|{\bf z} \| \leq R$), and $\mathcal{A}$ be the set of all binary inference classifiers taking ${\bf z}$ as an input.
Assume the task classifier $C$ is $C_L$-Lipschitz,
i.e., $\|C\|_L \leq C_L$. Then, we have the below relationship between the standard risk and the advantage:  
{
\begin{align}
\label{eqn:uptradeoff}
    \mathrm{Risk}(C \circ f)
    & \geq \Delta_{y|u} - 2R \cdot C_L \cdot \textrm{Adv}_{\mathcal{D}}(\mathcal{A}),
\end{align}
}
where $\Delta_{y|u} = |\textrm{Pr}_{\mathcal{D}}(y=1|u=0) - \textrm{Pr}_{\mathcal{D}}(y=1|u=1)|$ is a  dataset-dependent constant. 
\end{restatable}

\noindent \emph{Remark.} 
Theorem~\ref{thm:uptradeoff} says that any task classifier using learnt representations  leaks attribute privacy: the smaller the advantage $\textrm{Adv}_{\mathcal{D}}(\mathcal{A})$ (meaning less  attribute privacy is leaked), the larger the  lower bound risk, and vice versa. Note that the lower bound is {independent} of the adversary, meaning it covers the \emph{worst-case}  attribute inference adversary. Hence, Equation (\ref{eqn:uptradeoff}) reflects an inherent tradeoff between utility preservation and attribute privacy leakage. 

\noindent {\bf Robustness vs. Privacy Tradeoff.}
Let ${\mathcal{D}'}$ be a joint distribution over the adversarially perturbed input ${\bf x}'$, sensitive attribute $u$, and label $y$. 
By assuming the representation space
is bounded by $R$, 
the perturbed representations also satisfy $\max_{{\bf z}' \in \mathcal{Z}} \|{\bf z}' \| \leq R$, where ${\bf z}' = f({\bf x}')$. 
Following Equation~\ref{eqn:adv}, we have an associated adversary \emph{advantage} $\textrm{Adv}_{\mathcal{D}'} (\mathcal{A})$ with respect to 
the joint distribution 
$\mathcal{D}'$. 
Similarly, $\textrm{Adv}_{\mathcal{D}'}(\mathcal{A}) =1$  means an adversary can \emph{completely} infer the privacy attribute $u$ through the learnt adversarially perturbed representations ${\bf z}'$, and $\textrm{Adv}_{\mathcal{D}'}(\mathcal{A}) =0$ implies an adversary only obtains a \emph{random guessing} inference performance.  
Then we have the following theorem: 

\begin{restatable}[]{theorem}{rptradeoff}
\label{thm:rptradeoff}
Let ${\bf z}' = f ({\bf x}')$ be the learnt representation for ${\bf x}' \in \mathcal{B}({\bf x}, \epsilon)$ with a bounded norm $R$ 
(i.e., $\max_{{\bf z}' \in \mathcal{Z}} \|{\bf z}' \| \leq R$), and $\mathcal{A}$ be the set of all binary inference classifiers. 
Under a $C_L$-Lipschitz task classifier $C$, we have the below relationship between the adversarial risk 
and the advantage: 
{
\begin{align}
\label{eqn:rptradeoff}
    \mathrm{AdvRisk}_\epsilon (C \circ f)
    & \geq \Delta_{y|u} - 2R \cdot C_L \cdot \textrm{Adv}_{\mathcal{D}'}(\mathcal{A}).
\end{align}
}
\end{restatable}

\noindent \emph{Remark.} 
Similar to Theorem~\ref{thm:uptradeoff},  Theorem~\ref{thm:rptradeoff} states that, any task classifier using adversarially learnt representations has to incur an adversarial risk on at least a private attribute value. 
Moreover, the lower bound  covers the \emph{worst-case} adversary. Equation (\ref{eqn:rptradeoff})  hence reflects an inherent trade-off between adversarial robustness and privacy. 

\noindent {\bf Guaranteed Attribute Privacy Leakage.}
The attribute inference accuracy induced by the worst-case adversary is bounded in the following theorem:

\begin{restatable}[]{theorem}{provprivacy}
\label{thm:provprivacy}
Let ${\bf z}$ be the learnt  representation 
by  Equation (\ref{eq:final-obj}). 
For any attribute inference adversary $\mathcal{A}= \{A: 
\mathcal{Z} \rightarrow \mathcal{U} = \{0,1\} \}$, 
$\textrm{Pr}({A}({\bf z}) = u) \leq 1 - \frac{H(u | {\bf z})}{2 \log_2 ({6}/{H (u | {\bf z}}))}$. 
\end{restatable}

\noindent \emph{Remark.} Theorem~\ref{thm:provprivacy} shows when $ H(u | {\bf z})$ is larger, the inference accuracy induced by any adversary is smaller, i.e., less attribute privacy leakage. 
From another perspective, as $H(u | {\bf z}) = H(u) - I(u;  {\bf z})$, achieving the largest $H(u | {\bf z})$ implies minimizing $I(u;  {\bf z})$ (note that $H(u)$ is a constant)---This is exactly our {\bf Goal 1} aims to achieve. 

\begin{figure*}[t]
	\centering
 \subfloat[{CelebA (Gender): R.}]
	{\centering\includegraphics[scale=0.2]{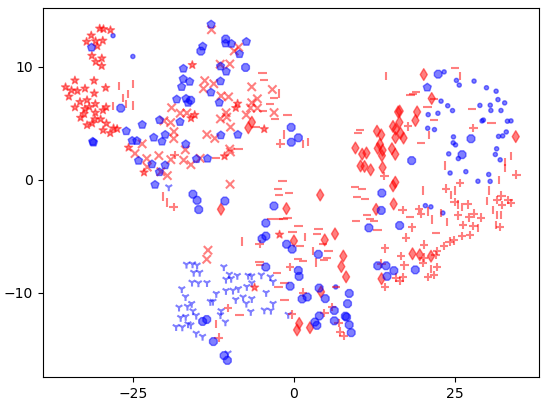}}
 \,
	\subfloat[{CelebA (Gender): R.+P.P.}]
	{\centering \includegraphics[scale=0.2]{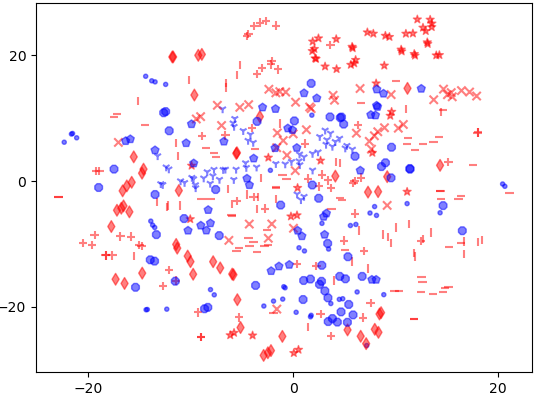}}
\,
 \subfloat[{Loans (Race): R.}]
	{\centering\includegraphics[scale=0.2]{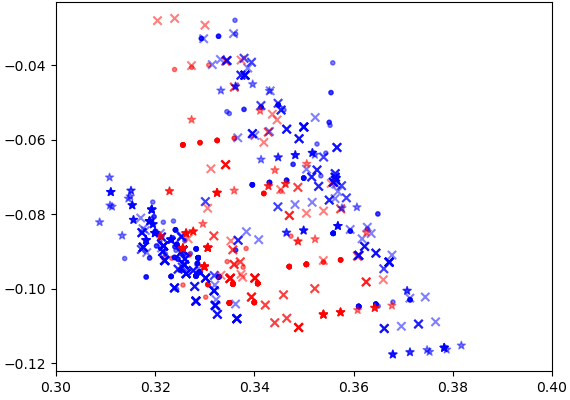}}
 \,
	\subfloat[{Loans (Race):R.+P.P.}]
	{\centering \includegraphics[scale=0.2]{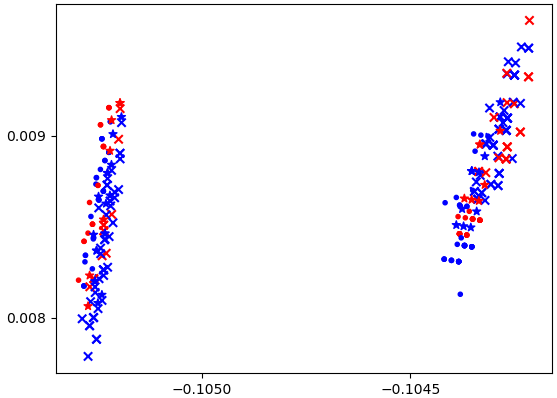}}  
\,	
\\
 \subfloat[{Adult (Gender): R.}]
	{\centering\includegraphics[scale=0.2]{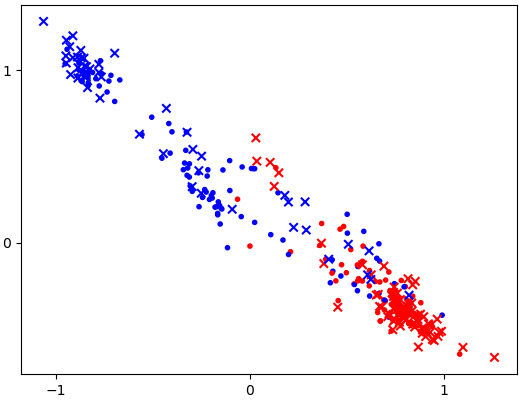}}
\,	
 \subfloat[{Adult (Gender): R.+P.P.}]
	{\centering \includegraphics[scale=0.2]{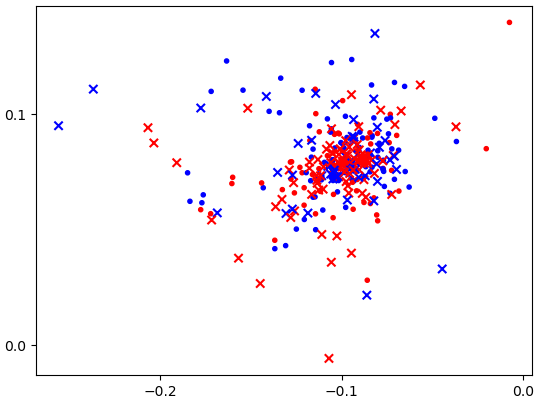}}
\quad  
	\subfloat[{Adult (Marital): R.}]
	{\centering\includegraphics[scale=0.2]{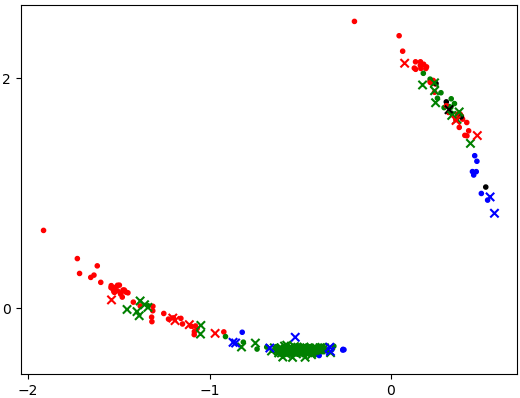}}
\quad
	\subfloat[{Adult (Marital): R.+P.P.}]
	{\centering \includegraphics[scale=0.2]{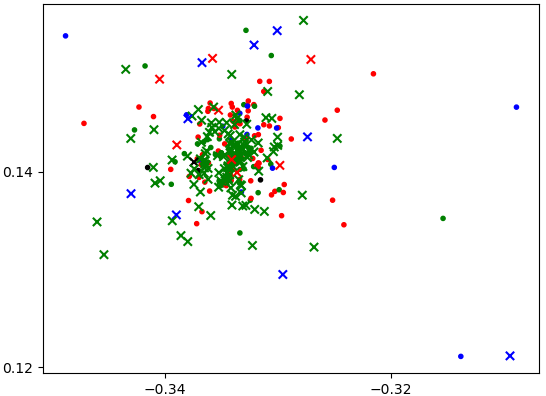}}
	\caption{2D t-SNE representations learnt by AdvPPRL. \emph{Left:} only robust representations; \emph{Right:} robust + privacy preserving representations (under the best tradeoff in Table~\ref{tab:allresults}). Colors indicate attribute values, while point patterns mean labels.} 
	\label{fig:TSNE_real}
\end{figure*}

\section{Evaluations}
\label{sec:eval}
We evaluate ARPRL on both synthetic and real-world datasets. The results on the synthetic dataset is for visualization and verifying the tradeoff purpose. 

\subsection{Experimental Setup}
We train the neural networks via Stochastic Gradient Descent (SGD), where the batch size is 100 and we use 10 local epochs and 50 global epochs in all datasets. The learning rate in SGD is set to be $1e^{-3}$. The detailed network architecture is shown in Table 2 in Appendix D
The hyperparameters used in the adversarially robust network are following \cite{zhu2020learning}. We also discuss how to choose the hyperparameters $\alpha$ and $\beta$ in real-world datasets in Appendix D. W.l.o.g, we consider the most powerful $l_\infty$ perturbation. 
Following~\cite{zhu2020learning}, we use the PGD attack for both generating adversarial perturbations in the estimation of worst-case MI and evaluating model robustness\footnote{
Note our goal is not to design the best adversarial attack, i.e., generating the optimal adversarial perturbation. Hence, the achieved adversarial robustness might not the optimal. 
We also test CelebA against the CW attack~\cite{CarliniSP17}, and the robust accuracy is 85\%, which close to 87\% with the PGD attack.
}. 
We implement ARPRL in PyTorch and use the NSF Chameleon Cloud GPUs \cite{keahey2020lessons} ({CentOS7-CUDA 11 with Nvidia Rtx 6000}) 
to train the model. 
We evaluate ARPRL on three metrics: utility preservation, adversarial robustness, and privacy protection. 

\subsection{Results on A Toy Example}
\label{sec:toy_exp}
We generate 2 2D circles with the center (0, 0) and (1, 0) respectively and radius 0.25, and data points are on the circumference. 
Each circle indicates a class and has 5,000 samples, where $80\%$ of the samples are for training and the rest $20\%$ for testing. We define the binary 
attribute value for each data as whether its $y$-value is above/below the $x$-axis. The network architecture is shown in Table D 
in  Appendix. 
 We use an $l_\infty$ perturbation budget $\epsilon=0.01$ and 10 PGD attack steps with
step size 0.1. 
We visualize the learnt representations via 2D t-SNE~\cite{van2008visualizing} in Figure~\ref{fig:toy2D}. 

We can see that: by learning  \emph{only robust} representations, the 2-class data can be well separated, but their private attribute values can be also completely separated--almost $100\%$ privacy leakage. 
In contrast, by learning both \emph{robust and privacy-preserving} representations,  
the 2-class data can be separated, but their private attributes are mixed---only 55\% inference accuracy. Note that the optimal random guessing inference accuracy is 50\%. 
We also notice a tradeoff among robustness/utility and attribute privacy, as demonstrated in our theorems. That is, a more robust/accurate model leaks more attribute privacy, and vice versa. 

\subsection{Results on the Real-World Datasets}
\label{sec:eval_real}
\noindent {\bf Datasets and setup.} We use three real-world datasets from different applications, i.e., the 
widely-used CelebA~\citep{liu2015deep} image dataset 
(150K training images and 50K for testing),
the Loans \cite{10.5555/3157382.3157469}, and Adult Income~\cite{Dua:2019} datasets 
In CelebA, we treat binary `gender' as the private attribute, and detect `gray hair' as the primary (binary classification) task, following \cite{li2021deepobfuscator,osia2018deep}.
For the Loans dataset, the primary task is to predict the affordability of the person asking for the loan while protecting their race. 
 For the Adult Income dataset, predicting whether the income of a person is above \$50K or not is the primary task. The private attributes are the gender and marital status.   
For $l_\infty$ perturbations, we set the budget $\epsilon$ = 0.01 on Loans and Adults, and 0.1 on 
CelebA. We use 10 PGD attack steps with step size 0.1. 

\noindent {\bf Results.} Tables~\ref{tab:allresults} shows the results 
on the three datasets, where we report the robust accuracy (under the $l_\infty$ attack), normal test accuracy, and  attribute inference accuracy (as well as the gap to random guessing). We have the following observations: 1) When $\alpha=0$, it means ARPRL only focuses on learning robust representation  (similar  to~\cite{zhu2020learning}) and obtains the best robust accuracy. However, the inference accuracy is rather high, indicating a serious privacy leakage. 
2) Increasing $\alpha$ can progressively better protect the attribute privacy, i.e., the inference accuracy is gradually reduced and finally close to random guessing (note different datasets have different random guessing value). 
3) $\alpha$ and $\beta$ together act as the tradeoff among robustness, utility, and privacy. Particularly, a better privacy protection (i.e., larger $\alpha$) implies a smaller test accuracy, indicating an utility and privacy tradeoff, as validated in Theorem~\ref{thm:uptradeoff}. Similarly,  a better privacy protection also implies a 
smaller robust accuracy, indicating a robustness and privacy tradeoff, as validated in Theorem~\ref{thm:rptradeoff}.

\noindent {\bf Visualization.}  We further visualize the learnt representations via t-SNE in Figure~\ref{fig:TSNE_real}. We can see that: When only focusing on learning robust representations, both the data with different labels and with different attribute values can be well separated.   
On the other hand, when learning both robust and privacy-preserving representations, the data with different labels can be separately, but they are mixed in term of the  attribute values---meaning the privacy of attribute values is protected to some extent. 

\noindent {\bf Runtime.} We only show runtime on the largest CelebA (150K training images). In our used platform, it took about 5 mins each epoch (about 15 hours in total) to learn the robust and privacy-preserving representation for each hyperparameter setting. The computational bottleneck is mainly from training robust representations (where we adapt the source code from \cite{zhu2020learning}), which occupies 60\% of the training time (e.g., 3 mins out of 5 mins in each epoch). Training the other neural networks is much faster. 
\begin{table*}[!tbh]
\centering
   \caption{Test accuracy, robust accuracy, vs. inference accuracy (and gap w.r.t. the optimal random guessing) on the considered three datasets and private attributes. 
   Note that some datasets are unbalanced, so the random guessing values are different.  
   Larger $\alpha$ means more privacy protection, while larger $\beta$ means more robust against adversarial perturbation.  
   $\alpha=0$ means no privacy protection and only focuses on robust representation learning, same as~\cite{zhu2020learning,zhou2022improving}.}
\begin{minipage}{.49\linewidth}
\centering
    \begin{tabular}{l|c|c|c|c}
           \multicolumn{5}{c}{\bf CelebA}\\
          \hline
           \multicolumn{5}{c}{Private attr.: Gender (binary), budget $\epsilon=0.1$}\\
              \hline
     \textbf{$\alpha$} & \textbf{$\beta$} & {Rob. Acc} & {Test Acc} & {Infer. Acc (gap) } \\
     \hline
     0 & 0.50 &  0.87 &  0.91 & 0.81 (0.31)\\
     0.1 & 0.45 & 0.84 & 0.88 & 0.75 (0.25)\\
     0.5 & 0.25 & 0.79 & 0.85 & 0.62 (0.12)\\
     0.9 & 0.05 & 0.71 & 0.81 & 0.57 (0.07)\\
   \end{tabular} 
   \end{minipage}
    \begin{minipage}{.49\linewidth}
    \centering
 \begin{tabular}{l|c|c|c|c}
     \multicolumn{5}{c}{\bf Loans}\\
      \hline
      \multicolumn{5}{c}{Private attr.: Race (binary), budget $\epsilon=0.01$}\\
     \hline
     \textbf{$\alpha$} & \textbf{$\beta$} & {Rob. Acc} & {Test Acc} & {Infer. Acc (gap) } \\
     \hline
     0 & 0.50 & 0.45 & 0.74 & 0.92 (0.22) \\
     0.05 & 0.475 & 0.42 & 0.69 & 0.75 (0.05)\\
     0.10 & 0.45  & 0.40 & 0.68 & 0.72 (0.02)\\
     0.15 & 0.425 & 0.39 &  0.66 & 0.71 (0.01)\\
   \end{tabular}
     \end{minipage}
     \begin{minipage}{.49\linewidth}
     \centering
     \begin{tabular}{l|c|c|c|c}
     \hline
     \multicolumn{5}{c}{\bf Adult income}\\
   \hline
      \textbf{$\alpha$} & \textbf{$\beta$} & {Rob. Acc} & {Test Acc} & {Infer. Acc (gap) } \\
     \hline
     \multicolumn{5}{c}{Private attr.: Gender (binary), budget $\epsilon=0.01$}\\
     \hline
     0 & 0.5 & 0.63 &   0.68 & 0.88 (0.33)\\
     0.05 & 0.475 & 0.57 &   0.67 & 0.72 (0.17)\\
     0.10 & 0.45 & 0.55  &   0.65 & 0.59 (0.04) \\
     0.20 & 0.4 & 0.53  & 0.63 & 0.55 (0.00)\\
     \hline
   \end{tabular}
   \end{minipage}
   \begin{minipage}{.49\linewidth}
   \centering
     \begin{tabular}{l|c|c|c|c}
     \hline
     \multicolumn{5}{c}{\bf Adult income}\\
   \hline
      \textbf{$\alpha$} & \textbf{$\beta$} & {Rob. Acc} & {Test Acc} & {Infer. Acc (gap) } \\
     \hline
     \multicolumn{5}{c}{Private attr.: Marital status (7 values), budget $\epsilon=0.01$}\\
     \hline
          0 & 0.5 &  0.56 &   0.71 & 0.70 (0.14)\\
     0.001 & 0.495 &   0.55 &   0.65 & 0.60 (0.04)\\
     0.005 & 0.49 &  0.52 &   0.60 & 0.59 (0.03)\\
     0.01 & 0.45 & 0.47 &   0.59 & 0.57 (0.01)\\
     \hline
   \end{tabular}
   \end{minipage}
   \label{tab:allresults}
\end{table*}
\subsection{Comparing with the State-of-the-arts}
{\bf Comparing with task-known privacy-protection baselines.} 
We compare ARPRL with two recent task-known methods for attribute privacy protection on CelebA: {\bf DPFE} \cite{osia2018deep} that also uses mutual information (but in different ways), and {\bf Deepobfuscator} \cite{li2021deepobfuscator}\footnote{We observe that a most recent work \cite{pmlr-v216-jeong23a} has similar performance as Deepobfuscator, but 2 orders of memory consumption. We do not include their results for conciseness.}, an adversarial training based defense. We align three methods with the same test accuracy 0.88, and compare attribute inference accuracy. \emph{For fair comparison, we do not consider adversarial robustness in our ARPRL.} The attribute inference accuracy of {DPFE}  and {Deepobfuscator} are 0.79 and 0.70, respectively, and our ARPRL’s is 0.71.
First, {DPFE} performs much worse because it assumes the distribution of the learnt representation to be Gaussian (which could be inaccurate), while {Deepobfuscator} and {ARPRL} have no assumption on the distributions; Second, 
{Deepobfuscator}  performs slightly better than ARPRL. This is
because both ARPRL and {Deepobfuscator} involve adversarial training, {Deepobfuscator} uses task labels, but ARPRL is task-agnostic, hence slightly sacrificing privacy. 

\noindent {\bf Comparing with task-known adversarial robustness baselines.}
We compare ARPRL with the state-of-the-art task-known adversarial training based TRADES~\cite{zhang2019theoretically} and test on CelebA, under the same adversarial perturbation and without privacy-protection (i.e., $\alpha=0$). 
For task-agnostic ARPRL, its robust accuracy is 0.87, which is slightly worse than TRADES's is 0.89. 
However, when ARPRL also includes task labels during training, its robust accuracy increases to 0.91---This again verifies that adversarially robust representations based defenses outperform the classic adversarial training based method. 

\noindent {\bf Comparing with task-known TRADES + Deepobfuscator for both robustness and privacy protection.} A natural solution to achieve both robustness and privacy protection is by combining SOTAs that are individually adversarially robust or privacy-preserving. We test TRADES + Deepobfuscator on CelebA. Tuning the tradeoff hyperparameters, we obtain the best utility, privacy, and robustness tradeoff at: (Robust Acc, Test Acc, Infer. Acc) = (0.79, 0.84, 0.65) while the best tradeoff of ARPRL in Table~\ref{tab:allresults} is (Robust Acc, Test Acc, Inference Acc) = (0.79, 0.85, 0.62), which is slightly better than TRADES + Deepobfuscator, though they both know the task labels. The results imply that simply combining SOTA robust and privacy-preserving methods is not the best option. Instead, our ARPRL learns both robust and privacy-preserving representations under the same information-theoretic framework.

\section{Related Work}
\label{sec:related}
\noindent {\bf Defenses against adversarial examples.} Many efforts have been made to improve the adversarial robustness of ML models against adversarial examples~\citep{kurakin2017adversarial,pang2019improving,zhang2019theoretically,wong2018provable,mao2019metric,cohen2019certified,wang2019improving,dong2020adversarial,lecuyer2019certified,zhai2020macer,wong2020fast,zhou2021towards,hong2022unicr,zhang2024text}.
Among them, adversarial training based defenses~\citep{madry2018towards,dong2020adversarial} has become the state-of-the-art.
At a high level, adversarial training augments training data with adversarial examples, e.g., via CW attack~\citep{CarliniSP17}, PGD attack~\citep{madry2018towards}, AutoAttack~\citep{croce2020reliable}), and uses a min-max formulation to train the target ML model~\citep{madry2018towards}. 
However, as pointed out by~\citep{zhu2020learning,zhou2022improving}, the dependence between the output of the target model and the input/adversarial examples has not been well studied, making the ability of adversarial training not fully exploited. To improve it, they 
propose to learn adversarially-robust representations via mutual information, which is shown to outperform the state-of-the-art adversarial training based defenses.  
Our ARPRL is inspired by them while having a nontrivial generalization to learn both robust and privacy-preserving representations with guarantees. 

\noindent {\bf Defenses against inference attacks.}
Existing defense methods against inference attacks can be roughly classified as 
\textit{adversarial learning}~\citep{oh2017adversarial,wu2018towards,pittaluga2019learning,liu2019privacy,pmlr-v216-jeong23a,xu2022neuguard}, 
\textit{differential privacy}~\citep{shokri2015privacy,abadi2016deep,feng2024universally}, 
and \textit{information obfuscation}~\citep{bertran2019adversarially,hamm2017minimax,osia2018deep,roy2019mitigating,zhao2020trade,azam2022can,xie2022differentially,varun2024towards}. 
Adversarial learning methods are inspired by GAN~\citep{goodfellow2014generative} and they learn features from training data so that their private information cannot be inferred from a learnt model. However, these methods need to know the primary task and lack of formal privacy guarantees.
Differential privacy methods have formal privacy guarantees, but they have high utility losses.
Information obfuscation methods aim to maximize the utility, under the constraint of bounding the information leakage, but almost all of them are empirical and task-dependent. 
{The only exception is \citep{zhao2020trade}, which has guaranteed information leakage.} However, this works requires stronger assumptions (e.g., conditional independence assumption between variables).  
Our work can be seen as a combination of information obfuscation with adversarial learning. 
It also offers privacy leakage guarantees and inherent trade-offs between robustness/utility and privacy.

\noindent {\bf Information-theoretical representation learning against inference attacks.} \citet{wang2021privacy} propose to use mutual information to learn privacy-preserving representation on graphs against node (or link) inference attacks, while keeping the primary link (or node) prediction performance. \citet{arevalo2024task} learns task-agnostic privacy-preserving representations for federated learning against attribute inference attacks with privacy guarantees. Further, \citet{noorbakhsh2024inf2guard} developed an information-theoretic framework (called Inf$^2$Guard) to defend against common inference attacks including membership inference, property inference, and data reconstruction, while offering privacy guarantees. 

\section{Conclusion}
\label{sec:conclusion}

We develop machine learning models to be robust against adversarial examples and protect sensitive attributes in training data. 
We achieve the goal by proposing ARPRL, which learns adversarially robust,  privacy preserving, and utility preservation representations formulated via mutual information. 
We also derive theoretical results that show the inherent tradeoff between robustness/utility and privacy and guarantees of attribute privacy against the worst-case
adversary.

\section*{\bf Acknowledgements} We sincerely thank the anonymous reviewers for their constructive feedback. 
This research was partially supported by the Cisco Research Award and the National Science Foundation under grant Nos. ECCS-2216926, CNS-2241713, CNS-2331302, CNS-2339686, CNS-2302689, and CNS-2308730, CNS-2319277, and CMMI-2326341.

\bibliography{ref}

\begin{thebibliography}{68}
\providecommand{\natexlab}[1]{#1}

\bibitem[{Abadi et~al.(2016)Abadi, Chu, Goodfellow, McMahan, Mironov, Talwar, and Zhang}]{abadi2016deep}
Abadi, M.; Chu, A.; Goodfellow, I.; McMahan, H.~B.; Mironov, I.; Talwar, K.; and Zhang, L. 2016.
\newblock Deep learning with differential privacy.
\newblock In \emph{CCS}.

\bibitem[{Alemi et~al.(2017)Alemi, Fischer, Dillon, and Murphy}]{alemi2017deep}
Alemi, A.~A.; Fischer, I.; Dillon, J.~V.; and Murphy, K. 2017.
\newblock Deep variational information bottleneck.
\newblock In \emph{ICLR}.

\bibitem[{Aono et~al.(2017)Aono, Hayashi, Wang, and Moriai}]{aono2017privacy}
Aono, Y.; Hayashi, T.; Wang, L.; and Moriai, S. 2017.
\newblock Privacy-preserving deep learning: Revisited and enhanced.
\newblock In \emph{ATIS}.

\bibitem[{Arevalo et~al.(2024)Arevalo, Noorbakhsh, Dong, Hong, and Wang}]{arevalo2024task}
Arevalo, C.~A.; Noorbakhsh, S.~L.; Dong, Y.; Hong, Y.; and Wang, B. 2024.
\newblock Task-Agnostic Privacy-Preserving Representation Learning for Federated Learning against Attribute Inference Attacks.
\newblock In \emph{AAAI}.

\bibitem[{Azam et~al.(2022)Azam, Kim, Hosseinalipour, Joe-Wong, Bagchi, and Brinton}]{azam2022can}
Azam, S.~S.; Kim, T.; Hosseinalipour, S.; Joe-Wong, C.; Bagchi, S.; and Brinton, C. 2022.
\newblock Can we generalize and distribute private representation learning?
\newblock In \emph{AISTATS}.

\bibitem[{Belghazi et~al.(2018)Belghazi, Baratin, Rajeshwar, Ozair, Bengio, Courville, and Hjelm}]{belghazi2018mutual}
Belghazi, M.~I.; Baratin, A.; Rajeshwar, S.; Ozair, S.; Bengio, Y.; Courville, A.; and Hjelm, D. 2018.
\newblock Mutual information neural estimation.
\newblock In \emph{ICML}.

\bibitem[{Bertran et~al.(2019)Bertran, Martinez, Papadaki, Qiu, Rodrigues, Reeves, and Sapiro}]{bertran2019adversarially}
Bertran, M.; Martinez, N.; Papadaki, A.; Qiu, Q.; Rodrigues, M.; Reeves, G.; and Sapiro, G. 2019.
\newblock Adversarially learned representations for information obfuscation and inference.
\newblock In \emph{ICML}.

\bibitem[{Bortsova et~al.(2021)Bortsova, Gonz{\'a}lez-Gonzalo, Wetstein, Dubost, Katramados, Hogeweg, Liefers, van Ginneken, Pluim, Veta et~al.}]{bortsova2021adversarial}
Bortsova, G.; Gonz{\'a}lez-Gonzalo, C.; Wetstein, S.~C.; Dubost, F.; Katramados, I.; Hogeweg, L.; Liefers, B.; van Ginneken, B.; Pluim, J.~P.; Veta, M.; et~al. 2021.
\newblock Adversarial attack vulnerability of medical image analysis systems: Unexplored factors.
\newblock \emph{Medical Image Analysis}.

\bibitem[{Calabro(2009)}]{calabro2009exponential}
Calabro, C. 2009.
\newblock \emph{The exponential complexity of satisfiability problems}.
\newblock University of California, San Diego.

\bibitem[{Carlini and Wagner(2017)}]{CarliniSP17}
Carlini, N.; and Wagner, D. 2017.
\newblock Towards Evaluating the Robustness of Neural Networks.
\newblock In \emph{IEEE S \& P}.

\bibitem[{Cheng et~al.(2020)Cheng, Hao, Dai, Liu, Gan, and Carin}]{cheng2020club}
Cheng, P.; Hao, W.; Dai, S.; Liu, J.; Gan, Z.; and Carin, L. 2020.
\newblock CLUB: A Contrastive Log-ratio Upper Bound of Mutual Information.
\newblock In \emph{ICML}.

\bibitem[{Cohen, Rosenfeld, and Kolter(2019)}]{cohen2019certified}
Cohen, J.~M.; Rosenfeld, E.; and Kolter, J.~Z. 2019.
\newblock Certified adversarial robustness via randomized smoothing.
\newblock In \emph{ICML}.

\bibitem[{Croce and Hein(2020)}]{croce2020reliable}
Croce, F.; and Hein, M. 2020.
\newblock Reliable evaluation of adversarial robustness with an ensemble of diverse parameter-free attacks.
\newblock In \emph{ICML}.

\bibitem[{Dong et~al.(2020)Dong, Deng, Pang, Zhu, and Su}]{dong2020adversarial}
Dong, Y.; Deng, Z.; Pang, T.; Zhu, J.; and Su, H. 2020.
\newblock Adversarial distributional training for robust deep learning.
\newblock In \emph{NeurIPS}.

\bibitem[{Dua and Graff(2017)}]{Dua:2019}
Dua, D.; and Graff, C. 2017.
\newblock {UCI} Machine Learning Repository.

\bibitem[{Eykholt et~al.(2018)Eykholt, Evtimov, Fernandes, Li, Rahmati, Xiao, Prakash, Kohno, and Song}]{eykholt2018robust}
Eykholt, K.; Evtimov, I.; Fernandes, E.; Li, B.; Rahmati, A.; Xiao, C.; Prakash, A.; Kohno, T.; and Song, D. 2018.
\newblock Robust physical-world attacks on deep learning visual classification.
\newblock In \emph{CVPR}.

\bibitem[{Feng et~al.(2024)Feng, Mohammady, Hong, Yan, Kundu, Wang, and Hong}]{feng2024universally}
Feng, S.; Mohammady, M.; Hong, H.; Yan, S.; Kundu, A.; Wang, B.; and Hong, Y. 2024.
\newblock Universally Harmonizing Differential Privacy Mechanisms for Federated Learning: Boosting Accuracy and Convergence.
\newblock \emph{arXiv}.

\bibitem[{Gibbs and Su(2002)}]{gibbs2002choosing}
Gibbs, A.~L.; and Su, F.~E. 2002.
\newblock On choosing and bounding probability metrics.
\newblock \emph{International statistical review}, 70(3): 419--435.

\bibitem[{Goodfellow et~al.(2014)Goodfellow, Pouget-Abadie, Mirza, Xu, Warde-Farley, Ozair, Courville, and Bengio}]{goodfellow2014generative}
Goodfellow, I.; Pouget-Abadie, J.; Mirza, M.; Xu, B.; Warde-Farley, D.; Ozair, S.; Courville, A.; and Bengio, Y. 2014.
\newblock Generative adversarial nets.
\newblock In \emph{NIPS}.

\bibitem[{Hamm(2017)}]{hamm2017minimax}
Hamm, J. 2017.
\newblock Minimax filter: Learning to preserve privacy from inference attacks.
\newblock \emph{JMLR}.

\bibitem[{Hardt, Price, and Srebro(2016)}]{10.5555/3157382.3157469}
Hardt, M.; Price, E.; and Srebro, N. 2016.
\newblock Equality of Opportunity in Supervised Learning.
\newblock In \emph{NIPS}.

\bibitem[{Hjelm et~al.(2019)Hjelm, Fedorov, Lavoie-Marchildon, Grewal, Bachman, Trischler, and Bengio}]{hjelm2019learning}
Hjelm, R.~D.; Fedorov, A.; Lavoie-Marchildon, S.; Grewal, K.; Bachman, P.; Trischler, A.; and Bengio, Y. 2019.
\newblock Learning deep representations by mutual information estimation and maximization.
\newblock In \emph{ICLR}.

\bibitem[{Hong, Wang, and Hong(2022)}]{hong2022unicr}
Hong, H.; Wang, B.; and Hong, Y. 2022.
\newblock Unicr: Universally approximated certified robustness via randomized smoothing.
\newblock In \emph{ECCV}.

\bibitem[{Hong et~al.(2024)Hong, Zhang, Wang, Ba, and Hong}]{hong2024certifiable}
Hong, H.; Zhang, X.; Wang, B.; Ba, Z.; and Hong, Y. 2024.
\newblock Certifiable Black-Box Attacks with Randomized Adversarial Examples: Breaking Defenses with Provable Confidence.
\newblock In \emph{CCS}.

\bibitem[{Jeong et~al.(2023)Jeong, Cho, Benz, and Kim}]{pmlr-v216-jeong23a}
Jeong, J.; Cho, M.; Benz, P.; and Kim, T.-h. 2023.
\newblock Noisy adversarial representation learning for effective and efficient image obfuscation.
\newblock In \emph{UAI}.

\bibitem[{Jia et~al.(2017)Jia, Wang, Zhang, and Gong}]{AttriInferWWW17}
Jia, J.; Wang, B.; Zhang, L.; and Gong, N.~Z. 2017.
\newblock AttriInfer: Inferring User Attributes in Online Social Networks Using Markov Random Fields.
\newblock In \emph{WWW}.

\bibitem[{Keahey et~al.(2020)Keahey, Anderson, Zhen, Riteau, Ruth, Stanzione, Cevik, Colleran, Gunawi, Hammock, Mambretti, Barnes, Halbach, Rocha, and Stubbs}]{keahey2020lessons}
Keahey, K.; Anderson, J.; Zhen, Z.; Riteau, P.; Ruth, P.; Stanzione, D.; Cevik, M.; Colleran, J.; Gunawi, H.~S.; Hammock, C.; Mambretti, J.; Barnes, A.; Halbach, F.; Rocha, A.; and Stubbs, J. 2020.
\newblock Lessons Learned from the Chameleon Testbed.
\newblock In \emph{USENIX ATC}.

\bibitem[{Kurakin, Goodfellow, and Bengio(2017)}]{kurakin2017adversarial}
Kurakin, A.; Goodfellow, I.; and Bengio, S. 2017.
\newblock Adversarial machine learning at scale.
\newblock In \emph{ICLR}.

\bibitem[{Lecuyer et~al.(2019)Lecuyer, Atlidakis, Geambasu, Hsu, and Jana}]{lecuyer2019certified}
Lecuyer, M.; Atlidakis, V.; Geambasu, R.; Hsu, D.; and Jana, S. 2019.
\newblock Certified robustness to adversarial examples with differential privacy.
\newblock In \emph{IEEE SP}.

\bibitem[{Li et~al.(2021)Li, Guo, Yang, and Chen}]{li2021deepobfuscator}
Li, A.; Guo, J.; Yang, H.; and Chen, Y. 2021.
\newblock Deepobfuscator: Adversarial training framework for privacy-preserving image classification.
\newblock \emph{arXiv}.

\bibitem[{Liao et~al.(2021)Liao, Zhao, Xu, Jaakkola, Gordon, Jegelka, and Salakhutdinov}]{liao2021information}
Liao, P.; Zhao, H.; Xu, K.; Jaakkola, T.; Gordon, G.~J.; Jegelka, S.; and Salakhutdinov, R. 2021.
\newblock Information obfuscation of graph neural networks.
\newblock In \emph{ICML}.

\bibitem[{Liu et~al.(2019)Liu, Du, Shrivastava, and Zhong}]{liu2019privacy}
Liu, S.; Du, J.; Shrivastava, A.; and Zhong, L. 2019.
\newblock Privacy Adversarial Network: Representation Learning for Mobile Data Privacy.
\newblock \emph{Proceedings of the ACM on Interactive, Mobile, Wearable and Ubiquitous Technologies}, 3(4): 1--18.

\bibitem[{Liu et~al.(2015)Liu, Luo, Wang, and Tang}]{liu2015deep}
Liu, Z.; Luo, P.; Wang, X.; and Tang, X. 2015.
\newblock Deep learning face attributes in the wild.
\newblock In \emph{Proceedings of the IEEE international conference on computer vision}, 3730--3738.

\bibitem[{Madry et~al.(2018)Madry, Makelov, Schmidt, Tsipras, and Vladu}]{madry2018towards}
Madry, A.; Makelov, A.; Schmidt, L.; Tsipras, D.; and Vladu, A. 2018.
\newblock Towards deep learning models resistant to adversarial attacks.
\newblock In \emph{ICLR}.

\bibitem[{Mao et~al.(2019)Mao, Zhong, Yang, Vondrick, and Ray}]{mao2019metric}
Mao, C.; Zhong, Z.; Yang, J.; Vondrick, C.; and Ray, B. 2019.
\newblock Metric learning for adversarial robustness.
\newblock \emph{Advances in Neural Information Processing Systems}, 32.

\bibitem[{McAllester and Stratos(2020)}]{mcallester2020formallimitationsmeasurementmutual}
McAllester, D.; and Stratos, K. 2020.
\newblock Formal Limitations on the Measurement of Mutual Information.
\newblock arXiv:1811.04251.

\bibitem[{Melis et~al.(2019)Melis, Song, De~Cristofaro, and Shmatikov}]{melis2019exploiting}
Melis, L.; Song, C.; De~Cristofaro, E.; and Shmatikov, V. 2019.
\newblock Exploiting unintended feature leakage in collaborative learning.
\newblock In \emph{IEEE SP}.

\bibitem[{Noorbakhsh et~al.(2024)Noorbakhsh, Zhang, Hong, and Wang}]{noorbakhsh2024inf2guard}
Noorbakhsh, S.~L.; Zhang, B.; Hong, Y.; and Wang, B. 2024.
\newblock Inf2Guard: An Information-Theoretic Framework for Learning Privacy-Preserving Representations against Inference Attacks.
\newblock In \emph{USENIX Security}.

\bibitem[{Nowozin, Cseke, and Tomioka(2016)}]{nowozin2016f}
Nowozin, S.; Cseke, B.; and Tomioka, R. 2016.
\newblock f-gan: Training generative neural samplers using variational divergence minimization.
\newblock In \emph{NIPS}.

\bibitem[{Oh, Fritz, and Schiele(2017)}]{oh2017adversarial}
Oh, S.~J.; Fritz, M.; and Schiele, B. 2017.
\newblock Adversarial image perturbation for privacy protection a game theory perspective.
\newblock In \emph{ICCV}.

\bibitem[{Osia et~al.(2018)Osia, Taheri, Shamsabadi, Katevas, Haddadi, and Rabiee}]{osia2018deep}
Osia, S.~A.; Taheri, A.; Shamsabadi, A.~S.; Katevas, K.; Haddadi, H.; and Rabiee, H.~R. 2018.
\newblock Deep private-feature extraction.
\newblock \emph{IEEE TKDE}.

\bibitem[{Pang et~al.(2019)Pang, Xu, Du, Chen, and Zhu}]{pang2019improving}
Pang, T.; Xu, K.; Du, C.; Chen, N.; and Zhu, J. 2019.
\newblock Improving adversarial robustness via promoting ensemble diversity.
\newblock In \emph{ICML}.

\bibitem[{Peng et~al.(2019)Peng, Kanazawa, Toyer, Abbeel, and Levine}]{peng2018variational}
Peng, X.~B.; Kanazawa, A.; Toyer, S.; Abbeel, P.; and Levine, S. 2019.
\newblock Variational discriminator bottleneck: Improving imitation learning, inverse rl, and gans by constraining information flow.
\newblock In \emph{ICLR}.

\bibitem[{Pittaluga, Koppal, and Chakrabarti(2019)}]{pittaluga2019learning}
Pittaluga, F.; Koppal, S.; and Chakrabarti, A. 2019.
\newblock Learning privacy preserving encodings through adversarial training.
\newblock In \emph{WACV}.

\bibitem[{Poole et~al.(2019)Poole, Ozair, Oord, Alemi, and Tucker}]{poole2019variational}
Poole, B.; Ozair, S.; Oord, A. v.~d.; Alemi, A.~A.; and Tucker, G. 2019.
\newblock On variational bounds of mutual information.
\newblock In \emph{ICML}.

\bibitem[{Qu, Li, and Wang(2023)}]{qu2023certified}
Qu, W.; Li, Y.; and Wang, B. 2023.
\newblock A Certified Radius-Guided Attack Framework to Image Segmentation Models.
\newblock In \emph{IEEE EuroSP}.

\bibitem[{Roy and Boddeti(2019)}]{roy2019mitigating}
Roy, P.~C.; and Boddeti, V.~N. 2019.
\newblock Mitigating information leakage in image representations: A maximum entropy approach.
\newblock In \emph{CVPR}.

\bibitem[{Salem et~al.(2023)Salem, Cherubin, Evans, K{\"o}pf, Paverd, Suri, Tople, and Zanella-B{\'e}guelin}]{salem2023sok}
Salem, A.; Cherubin, G.; Evans, D.; K{\"o}pf, B.; Paverd, A.; Suri, A.; Tople, S.; and Zanella-B{\'e}guelin, S. 2023.
\newblock SoK: Let the privacy games begin! A unified treatment of data inference privacy in machine learning.
\newblock In \emph{IEEE SP}.

\bibitem[{Shokri and Shmatikov(2015)}]{shokri2015privacy}
Shokri, R.; and Shmatikov, V. 2015.
\newblock Privacy-preserving deep learning.
\newblock In \emph{CCS}.

\bibitem[{Song, Shokri, and Mittal(2019{\natexlab{a}})}]{song2019membership}
Song, L.; Shokri, R.; and Mittal, P. 2019{\natexlab{a}}.
\newblock Membership inference attacks against adversarially robust deep learning models.
\newblock In \emph{SPW}.

\bibitem[{Song, Shokri, and Mittal(2019{\natexlab{b}})}]{song2019privacy}
Song, L.; Shokri, R.; and Mittal, P. 2019{\natexlab{b}}.
\newblock Privacy risks of securing machine learning models against adversarial examples.
\newblock In \emph{CCS}.

\bibitem[{Szegedy et~al.(2013)Szegedy, Zaremba, Sutskever, Bruna, Erhan, Goodfellow, and Fergus}]{szegedy2013intriguing}
Szegedy, C.; Zaremba, W.; Sutskever, I.; Bruna, J.; Erhan, D.; Goodfellow, I.; and Fergus, R. 2013.
\newblock Intriguing properties of neural networks.
\newblock \emph{arXiv}.

\bibitem[{Van~der Maaten and Hinton(2008)}]{van2008visualizing}
Van~der Maaten, L.; and Hinton, G. 2008.
\newblock Visualizing data using t-SNE.
\newblock \emph{JMLR}, 9(11).

\bibitem[{Varun et~al.(2024)Varun, Feng, Wang, Sural, and Hong}]{varun2024towards}
Varun, M.; Feng, S.; Wang, H.; Sural, S.; and Hong, Y. 2024.
\newblock Towards Accurate and Stronger Local Differential Privacy for Federated Learning with Staircase Randomized Response.
\newblock In \emph{CODASPY}.

\bibitem[{Wang et~al.(2021)Wang, Guo, Li, Chen, and Li}]{wang2021privacy}
Wang, B.; Guo, J.; Li, A.; Chen, Y.; and Li, H. 2021.
\newblock Privacy-preserving representation learning on graphs: A mutual information perspective.
\newblock In \emph{Proceedings of the 27th acm sigkdd conference on knowledge discovery \& data mining}, 1667--1676.

\bibitem[{Wang et~al.(2019)Wang, Zou, Yi, Bailey, Ma, and Gu}]{wang2019improving}
Wang, Y.; Zou, D.; Yi, J.; Bailey, J.; Ma, X.; and Gu, Q. 2019.
\newblock Improving adversarial robustness requires revisiting misclassified examples.
\newblock In \emph{ICLR}.

\bibitem[{Wong and Kolter(2018)}]{wong2018provable}
Wong, E.; and Kolter, Z. 2018.
\newblock Provable defenses against adversarial examples via the convex outer adversarial polytope.
\newblock In \emph{ICML}.

\bibitem[{Wong, Rice, and Kolter(2020)}]{wong2020fast}
Wong, E.; Rice, L.; and Kolter, J.~Z. 2020.
\newblock Fast is better than free: Revisiting adversarial training.
\newblock In \emph{ICLR}.

\bibitem[{Wu et~al.(2018)Wu, Wang, Wang, and Jin}]{wu2018towards}
Wu, Z.; Wang, Z.; Wang, Z.; and Jin, H. 2018.
\newblock Towards privacy-preserving visual recognition via adversarial training: A pilot study.
\newblock In \emph{ECCV}.

\bibitem[{Xie and Hong(2022)}]{xie2022differentially}
Xie, S.; and Hong, Y. 2022.
\newblock Differentially private instance encoding against privacy attacks.
\newblock In \emph{NAACL-W}.

\bibitem[{Xu et~al.(2022)Xu, Wang, Ran, Wen, and Venkitasubramaniam}]{xu2022neuguard}
Xu, N.; Wang, B.; Ran, R.; Wen, W.; and Venkitasubramaniam, P. 2022.
\newblock Neuguard: Lightweight neuron-guided defense against membership inference attacks.
\newblock In \emph{ACSAC}.

\bibitem[{Zhai et~al.(2020)Zhai, Dan, He, Zhang, Gong, Ravikumar, Hsieh, and Wang}]{zhai2020macer}
Zhai, R.; Dan, C.; He, D.; Zhang, H.; Gong, B.; Ravikumar, P.; Hsieh, C.-J.; and Wang, L. 2020.
\newblock MACER: Attack-free and Scalable Robust Training via Maximizing Certified Radius.
\newblock In \emph{ICLR}.

\bibitem[{Zhang et~al.(2019)Zhang, Yu, Jiao, Xing, El~Ghaoui, and Jordan}]{zhang2019theoretically}
Zhang, H.; Yu, Y.; Jiao, J.; Xing, E.; El~Ghaoui, L.; and Jordan, M. 2019.
\newblock Theoretically principled trade-off between robustness and accuracy.
\newblock In \emph{ICML}.

\bibitem[{Zhang et~al.(2024)Zhang, Hong, Hong, Huang, Wang, Ba, and Ren}]{zhang2024text}
Zhang, X.; Hong, H.; Hong, Y.; Huang, P.; Wang, B.; Ba, Z.; and Ren, K. 2024.
\newblock Text-crs: A generalized certified robustness framework against textual adversarial attacks.
\newblock In \emph{IEEE SP}.

\bibitem[{Zhao et~al.(2020)Zhao, Chi, Tian, and Gordon}]{zhao2020trade}
Zhao, H.; Chi, J.; Tian, Y.; and Gordon, G.~J. 2020.
\newblock Trade-offs and guarantees of adversarial representation learning for information obfuscation.
\newblock In \emph{NeurIPS}.

\bibitem[{Zhou et~al.(2021)Zhou, Liu, Han, Wang, Peng, and Gao}]{zhou2021towards}
Zhou, D.; Liu, T.; Han, B.; Wang, N.; Peng, C.; and Gao, X. 2021.
\newblock Towards defending against adversarial examples via attack-invariant features.
\newblock In \emph{ICML}.

\bibitem[{Zhou et~al.(2022)Zhou, Wang, Gao, Han, Wang, Zhan, and Liu}]{zhou2022improving}
Zhou, D.; Wang, N.; Gao, X.; Han, B.; Wang, X.; Zhan, Y.; and Liu, T. 2022.
\newblock Improving Adversarial Robustness via Mutual Information Estimation.
\newblock In \emph{ICML}.

\bibitem[{Zhu, Zhang, and Evans(2020)}]{zhu2020learning}
Zhu, S.; Zhang, X.; and Evans, D. 2020.
\newblock Learning adversarially robust representations via worst-case mutual information maximization.
\newblock In \emph{ICML}.

\end{thebibliography}

\newpage
\appendix

\section{Algorithm 1}
\begin{algorithm}
\small
\caption{Adversarially robust and privacy-preserving representation learning ({\bf ARPRL)}}
\label{alg:ARPRL}
\begin{flushleft}
{\bf Input:}  A dataset $\mathcal{D}=\{{\bf x}_i, y_i, u_i\}$, perturbation budget $\epsilon$, 
$\alpha, \beta \in [0,1]$, $\lambda>0$, 
learning rates $lr_1, lr_2, lr_3, lr_4, lr_5$; \\\#global epochs $I$, \#local gradient steps $J$;

{\bf Output:} Representation learner parameters $\Theta$.
\end{flushleft}

\begin{algorithmic}[1]
\STATE Initialize $\Theta,\Psi, \Omega, \Lambda$ for the representation learner $f_\Theta$, 
 privacy protection network $g_{\Phi}$, utility preservation network $h_\Omega$, and adversarially robust network $t_\Lambda$; 

\FOR{$t=0; t < T; t++$}
    \STATE $L_1 = \sum\nolimits_{i} CE(u_i,g_{\Psi}(f({\bf x}_i)))$;
    \STATE $L_2 = \frac{1}{|\mathcal{D}|} \sum_{i} t_{\Lambda}({\bf x}_i, f_\Theta({\bf x}_i), u_i) - \log [\frac{1}{|\mathcal{D}|} \sum_{i} e^{t_{\Lambda}({\bar{\bf x}}_i, f_\Theta({\bf x}_i), u_i)}]$;
    \STATE $L_3 = I^{(JS)}_{\Theta,\Omega} ({\bf x}; f({\bf x}),u)$;
    \FOR{$i=0; i < I; i++$}
        \FOR{$j=0; j < J; j++$}
        \STATE $\Psi\leftarrow \Psi - lr_1 \cdot 
        \frac{\partial L_1}{\partial\Psi}$; 
        \STATE $\Lambda \leftarrow \Lambda + lr_2 \cdot \frac{\partial L_2}{\partial\Lambda}$;
        \STATE $\Omega\leftarrow \Omega + lr_3 \cdot \frac{\partial L_3}{\partial \Omega}$;
         \ENDFOR
        \STATE $\Theta\leftarrow \Theta + lr_4 \cdot \frac{\partial (\alpha L_1 + \beta L_2 + (1-\alpha-\beta) L_3)
        }{\partial \Theta} $;
    \ENDFOR
\ENDFOR \\
\end{algorithmic}
\end{algorithm}

\section{Derivation Details in Section 3}
\label{app:derivations}

\noindent {\bf Deriving a upper bound of $I({\bf z};u)$:} By applying the vCLUB bound proposed in \citep{cheng2020club}, we have: 
\begin{align}
& I({\bf z};u) 
 \leq I_{vCLUB}({\bf z};u) \nonumber \\
& = \mathbb{E}_{p({\bf z}, u)} [\log q_{\Psi}(u|{\bf z}) ] - \mathbb{E}_{ p({\bf z}) p(u)} [\log q_{\Psi}(u|{\bf z}) ], 
\end{align}
where 
$q_{\Psi}(u | {\bf z})$ is an auxiliary posterior distribution of $p(u | {\bf z})$. and it needs to satisfy the condition: 
$$KL (p({\bf z}, u) ||  q_{\Psi} ({\bf z}, u)) \leq KL (p({\bf z})p(u) || q_{\Psi} ({\bf z}, u)).$$
To achieve this, we need to minimize:
\begin{align}
    & \min_{\Psi}  KL (p({\bf z}, u) ||  q_{\Psi} ({\bf z}, u)) \nonumber \\ 
    &= \min_{\Psi}  KL (p(u | {\bf z}) ||  q_{\Psi} (u | {\bf z})) \notag \\
    &= \min_{\Psi}  \mathbb{E}_{    p({\bf z}, u)} [\log p(u| {\bf z})] -  \mathbb{E}_{    p({\bf z}, u)} [\log q_{\Psi}(u| {\bf z}))] \notag \\ 
    & \Longleftrightarrow \max_{\Psi}  \mathbb{E}_{    p({\bf z}, u)} [\log q_{\Psi}(u| {\bf z})],  \label{Eqn:KL_conv} 
\end{align}
where we use that 
$\mathbb{E}_{    p({\bf z}, u)} [\log p(u| {\bf z})]$ 
is irrelevant to $\Psi$.   

Hence, minimizing $I_{vCLUB}({\bf z};u)$ reduces to maximizing $\mathbb{E}_{p({\bf z}, u)} [\log q_{\Psi}(u| {\bf z})]$. 

\noindent {\bf Deriving a lower bound of $I({\bf x} ; {\bf z} |u ) $:} From~\cite{nowozin2016f}, we know 
\begin{align}
    I({\bf x} ; {\bf z} |u ) 
    & = H({\bf x} | u) - H({\bf x} | {\bf z},u) \notag \\ 
    & = H({\bf x} | u) + \mathbb{E}_{p({\bf x} , {\bf z}, u)} [\log p({\bf x} | {\bf z}, u))] \notag \\
    & = H({\bf x} | u) + \mathbb{E}_{p({\bf x} , {\bf z}, u)} [\log q_\Omega ({\bf x} | {\bf z}, u))] \notag \\
    & \qquad + \mathbb{E}_{p({\bf x} , {\bf z}, u)} [KL(p(\cdot|{\bf z}, u) || q_\Omega (\cdot | {\bf z}, u))]  \notag \\
    & \geq H({\bf x} | u) + \mathbb{E}_{p({\bf x} , {\bf z}, u)} [\log q_\Omega ({\bf x} | {\bf z}, u)],  
\end{align}
where $q_{\Omega}$ is an \emph{arbitrary} auxiliary posterior distribution that aims to maintain the information ${\bf x}$ in the representation ${\bf z}$ conditioned on the private $u$.

\section{Proofs}
\label{app:proofs}

\subsection{Proof of Theorem \ref{thm:urtradeoff}}
\label{supp:urtradeoff}

\urtradeoff*

\begin{proof}
Replacing $I({\bf x}; {\bf z})$ and  $\mathrm{RV}_\epsilon (f)$ in Theorem~\ref{thm:rob_rep} with $I({\bf x}; {\bf z}|u)$ and $\mathrm{RV}_\epsilon (f|u)$, and setting $|\mathcal{Y}|=2$, we reach Theorem~\ref{thm:urtradeoff}.  
\end{proof}

\subsection{Proof of Theorem \ref{thm:uptradeoff}}
\label{supp:uptradeoff}

We first introduce the following definitions and lemmas that will be used to prove Theorem~\ref{thm:uptradeoff}.

\begin{definition}[Lipschitz function and Lipschitz norm]
We say a function $f: A \rightarrow \mathbb{R}^m$ is $L$-Lipschitz continuous, if for any $a, b \in A$, $\|f(a)-f(b)\| \leq L \cdot \|a-b\|$. Lipschitz norm of $f$, i.e., $\|f\|_L$, is defined as $\|f\|_L = \max \frac{\|f(a)-f(b)\|_L}{\|a-b\|_L}$.
\end{definition}

\begin{definition}[Total variance (TV) distance]
\label{def:TV}
Let $\mathcal{D}_1$ and $\mathcal{D}_2$ be two distributions over the same sample space $\Gamma$,  the TV distance between $\mathcal{D}_1$ and $\mathcal{D}_2$ is defined as: $d_{TV}(\mathcal{D}_1, \mathcal{D}_2) =  \max_{E \subseteq \Gamma} |\mathcal{D}_1 (E) - \mathcal{D}_2(E)|$.
\end{definition} 

\begin{definition}[1-Wasserstein distance]
\label{def:wassdis}
Let $\mathcal{D}_1$ and $\mathcal{D}_2$ be two distributions over the same sample space $\Gamma$, the 1-Wasserstein distance between $\mathcal{D}_1$ and $\mathcal{D}_2$ is defined as $W_1(\mathcal{D}_1, \mathcal{D}_2) = \max_{\|f\|_L \leq 1} |\int_{\Gamma} f d\mathcal{D}_1 - \int_{\Gamma} f d\mathcal{D}_2 |$, where $\|\cdot\|_L$  is the Lipschitz norm of a real-valued function.
\end{definition}

\begin{definition}[Pushforward distribution]
\label{def:pushforDist}
Let $\mathcal{D}$ be a distribution over a sample space and $g$ be a function of the same space. Then we call $g(\mathcal{D})$ the pushforward distribution of $\mathcal{D}$.
\end{definition} 

\begin{lemma}[Contraction of the 1-Wasserstein distance]
\label{lem:contWass}
Let $g$ be a function defined on a space and $L$ be constant such that $\|g\|_L \leq C_L$. 
For any distributions $\mathcal{D}_1$ and $\mathcal{D}_2$ over this space, 
$ W_1(g(\mathcal{D}_1), g(\mathcal{D}_2))  \leq C_L \cdot W_1(\mathcal{D}_1, \mathcal{D}_2)$.
\end{lemma}

\begin{lemma}[1-Wasserstein distance on Bernoulli random variables]
\label{lem:berprob}
Let $y_1$ and $y_2$ be two Bernoulli random variables with distributions $\mathcal{D}_1$ and $\mathcal{D}_2$, respectively. Then, 
$W_1(\mathcal{D}_1, \mathcal{D}_2)  = |\textrm{Pr}(y_1=1) - \textrm{Pr}(y_2=1)|$.
\end{lemma}

\begin{lemma}
[Relationship between the 1-Wasserstein distance and the TV distance~\citep{gibbs2002choosing}]
\label{lem:tvWass}
Let $g$ be a function defined on a norm-bounded space $\mathcal{Z}$, where $\max_{{\bf z} \in \mathcal{Z}} \|{\bf z} \| \leq R$, and $\mathcal{D}_1$ and $\mathcal{D}_1$ are two distributions over the space $\mathcal{Z}$.  Then $W_1(g(\mathcal{D}_1), g(\mathcal{D}_2)) \leq 2R \cdot d_{TV}(g(\mathcal{D}_1), g(\mathcal{D}_2))$. 
\end{lemma}

\begin{lemma}
[Relationship between the TV distance and advantage~\cite{liao2021information}]
\label{lem:tvAdv}
Given a binary attribute $u \in \{0,1\}$. 
Let $\mathcal{D}_{u=a}$ be the conditional data distribution of $\mathcal{D}$ given $u=a$ over a sample space $\Gamma$. Let $Adv_\mathcal{D}(\mathcal{A})$ be the advantage of adversary. 
Then for any function $f$, we have  
$d_{TV}(f(\mathcal{D}_{u=0}), f(\mathcal{D}_{u=1})) = \textrm{Adv}_{\mathcal{D}}(\mathcal{A})$.

\end{lemma}

We now prove Theorem~\ref{thm:uptradeoff}, which is restated as below: 

\uptradeoff*

\begin{proof}
We denote $\mathcal{D}_{u=a}$ as the conditional data distribution of $\mathcal{D}$ given $u=a$, and $\mathcal{D}_{y|u}$ as the conditional distribution of label $y$ given $u$. 
$cf$ is denoted as the (binary) composition function of  $c \circ f_{\Theta}$. 
As $c$ is binary task classifier on the learnt representations, it follows that the pushforward 
$cf(\mathcal{D}_{u=0})$ and $cf(\mathcal{D}_{u=1})$ induce two distributions over the binary label space  $\mathcal{Y} = \{0,1\}$. 
By leveraging the triangle inequalities of the 1-Wasserstein distance, we have

\begin{align}
& W_1 (\mathcal{D}_{y|u=0}, \mathcal{D}_{y|u=1}) 
\nonumber \\ 
& \leq W_1 (\mathcal{D}_{y|u=0}, cf(\mathcal{D}_{u=0})) + 
W_1 (cf(\mathcal{D}_{u=0}), cf(\mathcal{D}_{u=1})) \nonumber \\
& \qquad +  W_1 (cf(\mathcal{D}_{u=1}), \mathcal{D}_{y|u=1})
\label{eqn:keytri}
\end{align}

Using Lemma~\ref{lem:berprob} on Bernoulli random variables $y|u=a$: 
\begin{align}
\label{eqn:berprob}
& W_1 (\mathcal{D}_{y|u=0}, \mathcal{D}_{y|u=1}) \notag \\ 
& =  |\textrm{Pr}_{\mathcal{D}}(y=1|u=0) - \textrm{Pr}_{\mathcal{D}}(y=1|u=1) | \notag \\
& = \Delta_{y|u}.
\end{align}

Using Lemma~\ref{lem:contWass} on the contraction of the 1-Wasserstein distance and that $\|c\|_L \leq C_L$, we have 
\begin{align}
\label{eqn:contWass}
W_1 (cf(\mathcal{D}_{u=0}), cf(\mathcal{D}_{u=1})) \leq C_L \cdot W_1(f(\mathcal{D}_{u=0}), f(\mathcal{D}_{u=1})).
\end{align}

Using Lemma~\ref{lem:tvWass} with $\max_{{\bf z}} \|{\bf z}\| \leq R$ 
, we have 
\begin{align}
\label{eqn:tvWass}
& W_1(f(\mathcal{D}_{u=0}), f(\mathcal{D}_{u=1})) \notag \\ 
& \leq 2R \cdot d_{TV}(f(\mathcal{D}_{u=0}), f(\mathcal{D}_{u=1})) \notag \\ 
& = 2R \cdot \textrm{Adv}_{\mathcal{D}}(\mathcal{A}), 
\end{align}
where the last equation is based on Lemma~\ref{lem:tvAdv}. 

Combing Equations~\ref{eqn:contWass} and~\ref{eqn:tvWass}, we have 
$$W_1 (cf(\mathcal{D}_{u=0}), cf(\mathcal{D}_{u=1})) \leq 2R \cdot  C_L \cdot \textrm{Adv}_{\mathcal{D}}(\mathcal{A}).$$ 
Furthermore, using Lemma~\ref{lem:berprob} on Bernoulli random variables $y$ and $cf({\bf x})$, we have 
\begin{align}
& W_1 (\mathcal{D}_{y|u=a}, cf(\mathcal{D}_{u=a})) \notag \\
& = |\textrm{Pr}_{\mathcal{D}}(y=1 | u=a) - \textrm{Pr}_{\mathcal{D}}({cf({\bf x}})=1 | u=a)) | \nonumber \\
& = |\mathbb{E}_{\mathcal{D}}[y|u=a] - \mathbb{E}_{\mathcal{D}}[cf({\bf x})| u=a]| \nonumber \\
& \leq \mathbb{E}_{\mathcal{D}}[|y-cf({\bf x})| |u=a] \nonumber \\
& = \textrm{Pr}_{\mathcal{D}}(y \neq cf({\bf x}) | u=a) \nonumber \\
& = \textrm{Risk}_{u=a}(c \circ f). \label{eqn:CEloss}
\end{align}
Hence, $W_1 (\mathcal{D}_{y|u=0}, cf(\mathcal{D}_{u=0}))+ W_1 (\mathcal{D}_{y|u=1}, cf(\mathcal{D}_{u=1})) \leq \textrm{Risk}(c \circ f)$.

Finally, by combining Equations (\ref{eqn:keytri}) - (\ref{eqn:CEloss}), we have:  
\begin{align}
    \Delta_{y|u} & \leq \textrm{Risk}(c \circ f)+ 2R \cdot C_L \cdot \textrm{Adv}_{\mathcal{D}}(\mathcal{A}),
\end{align}
thus 
$
    \mathrm{Risk}(c \circ f) 
    \geq \Delta_{y|u} - 2R \cdot C_L \cdot \textrm{Adv}_{\mathcal{D}}(\mathcal{A}),
$ completing the proof.

\end{proof}

\subsection{Proof of Theorem \ref{thm:rptradeoff}}
\label{supp:rptradeoff}

We follow the way as proving Theorem~\ref{thm:uptradeoff}. 
We first restate Theorem~\ref{thm:rptradeoff} as below: 
\rptradeoff*

\begin{proof}
Recall that $D'$ is a joint distribution of the perturbed input ${\bf x}'$, the label $y$, and private attribute $u$. 
We denote $\mathcal{D}'_{u=a}$ as the conditional perturbed data distribution of $\mathcal{D}'$ given $u=a$, and $\mathcal{D}'_{y|u}$ as the conditional distribution of label $y$ given $u$. 
Also, the pushforward 
$cf(\mathcal{D}'_{u=a})$ induces two distributions over the binary label space $\mathcal{Y} = \{0,1\}$ with $a=\{0,1\}$. 
Via the triangle inequalities of the 1-Wasserstein distance, we have
\begin{align}
& W_1 (\mathcal{D}'_{y|u=0}, \mathcal{D}'_{y|u=1}) \notag \\
& \leq W_1 (\mathcal{D}'_{y|u=0}, cf(\mathcal{D}'_{u=0})) + 
W_1 (cf(\mathcal{D}'_{u=0}), cf(\mathcal{D}'_{u=1})) \notag \\
& \quad +  W_1 (cf(\mathcal{D}'_{u=1}), \mathcal{D}'_{y|u=1})
\label{eqn:keytri_w}
\end{align}
Using Lemma~\ref{lem:berprob} on Bernoulli random variables $y|u=a$:
\begin{small}
\begin{align}
\label{eqn:berprob_w}
W_1 (\mathcal{D}'_{y|u=0}, \mathcal{D}'_{y|u=1}) & =  |\textrm{Pr}_{\mathcal{D}'}(y=1|u=0) - \textrm{Pr}_{\mathcal{D}'}(y=1|u=1) | 
\notag \\ 
& =  |\textrm{Pr}_{\mathcal{D}}(y=1|u=0) - \textrm{Pr}_{\mathcal{D}}(y=1|u=1) | \notag \\
& = \Delta_{y|u},
\end{align}
\end{small}
where we use that the perturbed data  and clean data share the same label $y$ condition on $u$. 

Then following the proof of Theorem~\ref{thm:uptradeoff}, we have: 
\begin{align}
& W_1 (cf(\mathcal{D}'_{u=0}), cf(\mathcal{D}'_{u=1})) \leq C_L \cdot W_1(f(\mathcal{D}'_{u=0}), f(\mathcal{D}'_{u=1})); \label{eqn:contWass_adv} \\
& W_1(f(\mathcal{D}'_{u=0}), f(\mathcal{D}'_{u=1})) \leq 2R \cdot d_{TV}(f(\mathcal{D}'_{u=0}), f(\mathcal{D}'_{u=1})).
\end{align}

We further show $d_{TV}(f(\mathcal{D}'_{u=0}), f(\mathcal{D}'_{u=1})) = \textrm{Adv}_{\mathcal{D}'}(\mathcal{A})$:
\begin{align}
& d_{TV}(f(\mathcal{D}'_{u=0}), f(\mathcal{D}'_{u=1})) \notag \\ 
& = \max_{E} |\textrm{Pr}_{f(\mathcal{D}'_{u=0})}(E) - \textrm{Pr}_{f(\mathcal{D}'_{u=1})}(E)| \nonumber \\
& = \max_{A_E \in \mathcal{A}} | \textrm{Pr}_{{\bf z}' \sim f(\mathcal{D}'_{u=0})}(A_E({\bf z}')=1)  - \textrm{Pr}_{{\bf z}' \sim f(\mathcal{D}'_{u=1})}(A_E({\bf z})=1) | \nonumber \\
& = \max_{A_E \in \mathcal{A}} | \textrm{Pr}(A_E({\bf z}')=1 | u=0) - \textrm{Pr}(A_E({\bf z}')=1 | u=1) | \nonumber \\
& = \textrm{Adv}_{\mathcal{D}'}(\mathcal{A}), \label{eqn:tvadv}
\end{align}
where the first equation uses the definition of TV distance, and  $A_E(\cdot)$ is the characteristic function of the event $E$ in the second equation.

With Equations (\ref{eqn:contWass_adv}) - (\ref{eqn:tvadv}), we have 
$$W_1 (cf(\mathcal{D}'_{u=0}), cf(\mathcal{D}'_{u=1})) \leq 2R \cdot  C_L \cdot \textrm{Adv}_{\mathcal{D}'}(\mathcal{A}).$$ 
Furthermore, using Lemma~\ref{lem:berprob} on Bernoulli random variables $y$ and $cf({\bf x})$, we have 
\begin{align}
& W_1 (\mathcal{D}'_{y|u=0}, cf(\mathcal{D}'_{u=0})) +  W_1 (\mathcal{D}'_{y|u=1}, cf(\mathcal{D}'_{u=1})) \notag \\
& = |\textrm{Pr}_{\mathcal{D}'}(y=1 | u=0) - \textrm{Pr}_{\mathcal{D}'}({cf({\bf x}'})=1 | u=0)) |  \notag \\ 
& \quad + |\textrm{Pr}_{\mathcal{D}'}(y=1 | u=1) - \textrm{Pr}_{\mathcal{D}'}({cf({\bf x}'})=1 | u=1)) | \nonumber \\
& = |\mathbb{E}_{\mathcal{D}'}[y|u=0] - \mathbb{E}_{\mathcal{D}'}[cf({\bf x}')| u=0]| \notag \\ 
& \quad + |\mathbb{E}_{\mathcal{D}'}[y|u=1] - \mathbb{E}_{\mathcal{D}'}[cf({\bf x}')| u=1]| \nonumber \\
& \leq \mathbb{E}_{\mathcal{D}'}[|y-cf({\bf x}')| |u=0] + \mathbb{E}_{\mathcal{D}'}[|y-cf({\bf x}')| |u=1] \nonumber \\
& = \textrm{Pr}_{\mathcal{D}'}(y \neq cf({\bf x}') | u=0) + \textrm{Pr}_{\mathcal{D}'}(y \neq cf({\bf x}') | u=1) \nonumber \\
& = \textrm{Pr}_{\mathcal{D}'}(y \neq cf({\bf x}')) \nonumber \\
& = \textrm{Pr}_{\mathcal{D}} [\exists {\bf x}' \in \mathcal{B}({\bf x}, \epsilon), \, \textrm{s.t. } cf({\bf x}') \neq y] \notag \\ 
& = \textrm{AdvRisk}_\epsilon(c \circ f).  
\label{eqn:advrisk_tWass}
\end{align}

Finally, by combining Equations (\ref{eqn:keytri_w}) - (\ref{eqn:advrisk_tWass}), we have:  
\begin{align}
    \Delta_{y|u} & \leq \textrm{AdvRisk}_\epsilon(c \circ f) + 2R \cdot C_L \cdot \textrm{Adv}_{\mathcal{D}'}(\mathcal{A}) \notag
\end{align}

Hence, 
$
    \textrm{AdvRisk}_\epsilon(c \circ f) 
    \geq \Delta_{y|u} - 2R \cdot C_L \cdot \textrm{Adv}_{\mathcal{D}'}(\mathcal{A}),
$ completing the proof. 

\end{proof}

\subsection{Proof of Theorem \ref{thm:provprivacy}}
\label{supp:provprivacy}

We first point out that \cite{zhao2020trade} also provide the theoretical result in Theorem 3.1 against attribute inference attacks. However, there are two key differences between theirs and our Theorem~\ref{thm:provprivacy}: 
First, Theorem 3.1  requires an assumption $I(\hat{A}; A|Z)=0$, while our Theorem 5 does not need extra assumption; 2) The proof for Theorem 3.1 decomposes an joint entropy $H(A, \hat{A}, E)$, while our proof decomposes a conditional entropy $H(s, u | A(z))$. We note that the main idea to prove both theorems is by introducing an indicator and decomposing an entropy in two different ways. 

The following lemma about the inverse binary entropy will be useful in the proof of Theorem~\ref{thm:provprivacy}:
\begin{lemma}[\citep{calabro2009exponential} Theorem 2.2]
\label{lem:inveren}
Let $H_2^{-1}(p)$ be the inverse binary entropy function for $p \in [0,1]$, then $H_2^{-1}(p) \geq \frac{p}{2 \log_2(\frac{6}{p})}$.
\end{lemma}

\begin{lemma}[Data processing inequality] 
\label{lem:datainq}
Given random variables $X$, $Y$, and $Z$ that form a Markov chain in the
order $X \rightarrow Y \rightarrow Z$, then the mutual information between $X$ and $Y$ is greater than or equal to the mutual information between $X$ and $Z$. That is $I(X;Y) \geq I(X;Z)$.
\end{lemma}

With the above lemma, we are ready to prove Theorem~\ref{thm:provprivacy} restated as below. 
\provprivacy*

\begin{proof}
Let $s$ be an indicator that takes value 1 if and only if 
$\mathcal{A}({\bf z}) \neq u$, and 0 otherwise, i.e., $s = 1[\mathcal{A}({\bf z}) \neq u]$. Now consider the conditional entropy $H(s, u|\mathcal{A}({\bf z}))$ associated with $\mathcal{A}({\bf z})$, $u$, and $s$. By decomposing it in two different ways, we have 
\begin{align}
\label{eqn:entdecompo}
H(s, u|\mathcal{A}({\bf z})) & = H(u|\mathcal{A}({\bf z})) + H(s | u,  \mathcal{A}({\bf z})) \notag \\ 
& =  H(s | \mathcal{A}({\bf z})) + H(u| s, \mathcal{A}({\bf z})).
\end{align}

Note that $H(s | u, \mathcal{A}({\bf z}))=0$ as when $u$ and $\mathcal{A}({\bf z})$ are known, $s$ is also known. 
Similarly, 
\begin{align}
H(u| s, \mathcal{A}({\bf z})) & = Pr(s=1) H(u| s=1, \mathcal{A}({\bf z})) \notag \\
& \quad + Pr(s=0) H(u| s=0, \mathcal{A}({\bf z})) \notag \\ 
& = 0+0 = 0,
\end{align}
because when we know $s$'s value and $\mathcal{A}({\bf z})$, we also actually knows $u$.

Thus, Equation~\ref{eqn:entdecompo} reduces to $H(u|\mathcal{A}({\bf z}))=  H(s | \mathcal{A}({\bf z}))$. 
As conditioning does not increase entropy, i.e.,  
$H(s | \mathcal{A}({\bf z})) \leq H(s)$, we further have
\begin{align}
\label{eqn:finalone}
H(u|\mathcal{A}({\bf z})) \leq H(s).
\end{align}

On the other hand, using mutual information and entropy properties, we have 
$I(u; \mathcal{A}({\bf z})) = H(u) - H(u|\mathcal{A}({\bf z}))$ and
$I(u; {\bf z}) = H(u) - H(u|{\bf z})$. Hence, 
\begin{align}
\label{eqn:mi_en}
    I(u; \mathcal{A}({\bf z})) + H(u|\mathcal{A}({\bf z})) = I(u; {\bf z}) + H(u|{\bf z}).
\end{align}

Notice $\mathcal{A}({\bf z})$ is a random variable such that $u \perp \mathcal{A}({\bf z}) | {\bf z}$. Hence, we have the Markov chain $u \rightarrow {\bf z} \rightarrow \mathcal{A}({\bf z})$.
Based on the data processing inequality in Lemma~\ref{lem:datainq}, we know $I(u; \mathcal{A}({\bf z})) \leq I(u; {\bf z})$. 
Combining it with Equation~\ref{eqn:mi_en}, we have 
\begin{align}
\label{eqn:finaltwo}
H(u|\mathcal{A}({\bf z})) \geq H(u|{\bf z}).
\end{align}

Combing Equations (\ref{eqn:finalone}) and (\ref{eqn:finaltwo}), we have $H(s) = H_2(Pr(s=1) )\geq H(u|{\bf z})$,
which implies 
\begin{align*}
    \label{eqn:final}
    & Pr(\mathcal{A}({\bf z}) \neq u) = Pr(s=1) \geq H_2^{-1} (H(u|{\bf z})),
\end{align*}
where $H_2 (t) = -t \log_2 t - (1-t) \log_2 (1-t)$. 

Finally, by applying Lemma~\ref{lem:inveren}, we have 
$$Pr(\mathcal{A}({\bf z}) \neq u) \geq \frac{H(u|{\bf z})}{2 \log_2 ({6}/{H(u|{\bf z})})}.$$
Hence the attribute privacy leakage is bounded by $Pr(\mathcal{A}({\bf z}) = u) \leq 1 - \frac{H(u|{\bf z})}{2 \log_2 ({6}/{H(u|{\bf z})})}.$

\end{proof}

\subsection{Discussion on the Tightness of the MI Bound}

We leverage mutual information as a tool to learn robust and privacy-preserving representations as well as deriving theoretical results. However, computing the exact value of mutual information on high-dimensional feature spaces, as commonly encountered in deep learning, remains a computational challenge. Several works \cite{alemi2017deep,belghazi2018mutual,poole2019variational,hjelm2019learning,cheng2020club} have introduced various neural estimators to approximate mutual information by providing bounds. Unfortunately, as shown by \cite{mcallester2020formallimitationsmeasurementmutual}, these bounds are generally loose. At the moment, developing a method to achieve a tight mutual information bound, along with analyzing the associated approximation error, remains an open question in the field.

\section{Datasets and network architectures}
\label{appendix:Datasets}

\subsection{Detailed dataset descriptions}
\label{app:des}

\textbf{CelebA dataset~\cite{liu2015deep}.}
 CelebA consists of more than 200K face images with size 32x32. Each face image is labeled with 40 binary facial attributes. In the experiments, we use 150K images for training and 50K images for testing. 
 We treat binary ‘gender’ as the private attribute, and detect ‘gray hair’ as the primary (binary classification) task. 

\noindent \textbf{Loans dataset~\cite{10.5555/3157382.3157469}.}
This dataset is originally extracted from the loan-level Public Use Databases. The Federal Housing Finance Agency publishes these databases yearly, containing information about the Enterprises’ single family and multifamily mortgage acquisitions.  Specifically, the database used in this project is a single-family dataset and has a variety of features related to the person asking for a mortgage loan.  All the attributes in the dataset are numerical, so no preprocessing from this side was required. On the other hand, in order to create a balanced classification problem, some of the features were modified to have a similar number of observations belonging to all classes. We use 80\% data for training and 20\% for testing. 

The utility under this scope was measured in the system accurately predicting the affordability category of the person asking for a loan. This attribute is named \textit{Affordability}, and has three possible values: 0 if the person belongs to a mid-income family and asking for a loan in a low-income area, 1 if the person belongs to a low-income family and asking for a loan in a low-income area, and 2 if the person belongs to a low-income family and is asking for a loan not in a low-income area. The private attribute was set to be binary  \textit{Race}, being  White (0) or Not White (1).

\noindent \textbf{Adult Income dataset~\cite{Dua:2019}.}
This is a well-known dataset available in the UCI Machine Learning Repository. The dataset contains 32,561 observations each with 15 features, some of them numerical, other strings. Those attributes are not numerical were converted into categorical using an encoder. Again, we use the 80\%-20\% train-test split. 

The primary classification task is predicting if a person has an income above \$50,000, labeled as 1, or below, which is labeled as 0. The private attributes to predict are the \textit{Gender}, which is binary, and the \textit{Marital Status}, which has seven possible labels: 0 if Divorced, 1 if AF-spouse, 2 if Civil-spouse, 3 if Spouse absent, 4 if Never married, 5 if Separated, and 6 if Widowed. 

\subsection{Network architectures} 
\label{app:arch}

The used network architectures for the three neural networks are in Table~\ref{tab:arch}.  

\begin{table}[!t]
\caption{Network architectures for the used datasets. Note that utility preservation network is the same as robust  network.}
  \centering
    \begin{tabular}[t]{p{1.8cm}|p{1.8cm}|p{1.8cm}|p{1.5cm}}
      \hline
    \centering\textbf{Rep. Learner}&\centering\textbf{Robust Network} &\centering\textbf{Privacy Network} &\textbf{Utility Network}\\
    \hline
    \hline
    \multicolumn{4}{c}{CelebA}\\
    \hline
    \hline
    conv1-64 & conv3-64 & linear-32 & conv3-64\\ & & MaxPool \\
    \cline{1-4}
    conv64-128 & conv64-128 & linear-$\#$priv. attri. values & conv64-128\\
    \cline{1-4}
    linear-1024 & conv128-256 &  & conv128-256\\
    \cline{1-4}
    linear-64 & conv2048-2048 & & conv2048-2048\\
    & & & \\
    \cline{1-4}
    \hline
    \hline
    \multicolumn{4}{c}{Loans and Adult Income}\\
    \hline
    \hline
    linear-12  & linear-64 & linear-16 & linear-64  \\
    \cline{1-4}
    linear-3 & linear-3 & linear-$\#$priv. attri. values & linear-3 \\
    \cline{1-4}
    \hline
    \hline
    \multicolumn{4}{c}{Toy dataset}\\
    \hline
    \hline
    linear-10  & linear-64 & linear-5 & linear-64 \\
    \cline{1-4}
    linear-2 & linear-2 & linear-$\#$priv. attri. values & linear-2 \\
    \cline{1-4}
   \end{tabular}
   
   \label{tab:arch}
\end{table}

\subsection{How to Choose $\alpha$ and $\beta$}
\label{app:hyperpara}

Assume we reach the required utility with a (relatively large) value $1-\alpha-\beta$ (e.g., 0.7, 0.8; note its regularization controls the utility). Then we have a principled way to efficiently tune $\alpha$ and $\beta$ based on their meanings: 

1) We will start with three sets of $(\alpha1, \beta1), (\alpha2, \beta2), (\alpha3, \beta3)$, where one is with $\alpha1=\beta1$, one is with a larger $\alpha2 > \alpha1$ (i.e., better privacy), and one is with a larger $\beta3 > \beta1$ (better robustness), respectively. 

2) Based on the three results, we know whether a larger $\alpha$ or $\beta$ is needed to obtain a better privacy-robustness tradeoff and set their values via a binary search. For instance, if needing more privacy protection, we can set a larger $\alpha4 = \frac{\alpha1 + \alpha2}{2}$; or needing more robustness, we can set a larger $\beta4 = \frac{\beta1 + \beta3}{2}$. 

Step 2) continues until finding the optimal tradeoff $\alpha$ and $\beta$.

\end{document}